\documentclass{article}

\usepackage[accepted]{icml2023}

\usepackage[utf8]{inputenc} % allow utf-8 input
\usepackage[T1]{fontenc}    % use 8-bit T1 fonts
\usepackage{hyperref}       % hyperlinks
\usepackage{url}            % simple URL typesetting
\usepackage{array,booktabs}       % professional-quality tables
\usepackage{amsfonts}       % blackboard math symbols
\usepackage{pifont}
\usepackage{nicefrac}       % compact symbols for 1/2, etc.
\usepackage{microtype}      % microtypography
\usepackage{cancel}

\usepackage{natbib}  
\definecolor{midnightblue}{HTML}{0059b3}
\definecolor{noonblue}{HTML}{e5eef7}
\definecolor{chromered}{HTML}{f14233}
\definecolor{olivedrab}{HTML}{6b8e23}

\usepackage{hyperref} 
            \hypersetup{
                    colorlinks=true,                
                    breaklinks=true,                
                    urlcolor= black,                
                    linkcolor= midnightblue,                
                    bookmarksopen=false,
                    citecolor=black,
                    linkcolor=black,
                    filecolor=black,
                    citecolor=midnightblue,
                    linkbordercolor=white,
}

\usepackage{lipsum}
\usepackage{nicefrac}       % compact symbols for 1/2, etc.
\usepackage{microtype}      % microtypography
\usepackage{xcolor}
\usepackage{graphicx}
\usepackage{subfigure}
\usepackage{booktabs} 
\usepackage{times}    
\usepackage{amssymb,amsmath,amscd,amsfonts,amsthm,bbm,mathrsfs,yhmath}
\usepackage{bm}

\usepackage{caption}
\usepackage{graphics}
\usepackage{mathtools}
\usepackage[capitalize,noabbrev]{cleveref}
\usepackage{comment}

\usepackage{paralist,enumitem}                                       
\usepackage{pdfpages}

\usepackage[titletoc,toc,page]{appendix}
\usepackage{minitoc}

\newtheorem{theorem}{Theorem}
\newtheorem{lemma}{Lemma}
\newtheorem{corollary}{Corollary}
\newtheorem{proposition}{Proposition}
\newtheorem{definition}{Definition}
\newtheorem{remark}{Remark}

\usepackage[colorinlistoftodos,bordercolor=orange,backgroundcolor=orange!20,linecolor=orange,textsize=scriptsize]{todonotes}

\newcommand{\rmd}{{\mathrm d}}

\newcommand{\cF}{{\mathcal F}}

\newcommand{\cB}{{\mathcal B}}
\newcommand{\cO}{{\mathcal O}}

\newcommand{\cL}{{\mathcal L}}
\newcommand{\cP}{{\mathcal P}}

\newcommand{\bG}{{\mathbf G}}

\newcommand{\KL}[1]{H_{\pi}\left( #1\right)}
\newcommand{\FS}[1]{J_{\pi}\left( #1\right)}
\newcommand{\Exp}[1]{\mathbb{E}\left[ #1 \right]}
\newcommand{\dif}[1]{\frac{\rmd #1}{\rmd t}}

\newcommand{\bH}{\mathbf H}

\newcommand{\NN}{\mathbb N}
\newcommand{\RR}{\mathbb R}
\newcommand{\EE}{\mathbb E}
\newcommand{\BB}{\mathbb B}

\newcommand{\cQ}{\mathcal{Q}}

\newcommand{\norm}[1]{\left\| #1 \right\|}
\newcommand{\normsq}[1]{\left\| #1 \right\|^2}
\newcommand{\inner}[2]{\left< #1 , #2 \right>}

\newcommand{\brr}[1]{\left( #1 \right)}   % brackets round
\newcommand{\brs}[1]{\left[ #1 \right]}  % brackets square
\newcommand{\brc}[1]{\left\{ #1 \right\}} % brackets curly

\begin{document}

\twocolumn[
\icmltitle{ELF: Federated Langevin Algorithms with Primal, Dual and Bidirectional Compression}



\icmlsetsymbol{equal}{*}

\begin{icmlauthorlist}
\icmlauthor{Avetik Karagulyan}{kaust}
\icmlauthor{Peter Richtárik}{kaust}
% \icmlauthor{Firstname3 Lastname3}{comp}
\end{icmlauthorlist}

\icmlaffiliation{kaust}{AI Initiative, King Abdullah University of Science and Technology, Thuwal, Saudi Arabia}

\icmlcorrespondingauthor{Avetik Karagulyan}{avetik.karagulyan@kaust.edu.sa}
% \icmlcorrespondingauthor{Peter Richtárik}{peter.richtarik@kaust.edu.sa}

\icmlkeywords{Machine Learning, ICML}

\vskip 0.3in
]

\printAffiliationsAndNotice{}  % leave blank if no need to mention equal contribution

  \begin{abstract}
    Federated sampling algorithms have recently gained great popularity in the community of machine learning and statistics. 
    This paper studies variants of such algorithms called Error Feedback Langevin algorithms (ELF). 
    In particular, we analyze the combinations of  EF21 and EF21-P with the federated Langevin Monte-Carlo. 
    We propose three algorithms: P-ELF, D-ELF, and B-ELF that use, respectively,  primal, dual, and bidirectional compressors. 
    We analyze the proposed methods under Log-Sobolev inequality and provide non-asymptotic convergence guarantees.

  \end{abstract}
  % \addtocontents{toc}{\protect\setcounter{tocdepth}{0}}
\section{Introduction}

  Sampling from high-dimensional distributions is of colossal importance in modern statistics and machine learning. 
  The problem is particularly relevant in Bayesian inference \cite{robert2007bayesian}, where one deals with high-dimensional distributions that are hard to sample from.
  This paper focuses on  the sampling from the posteriors that arise in Bayesian federated learning 
  \cite{kassab2020federated,pmlr-vono22a-fed-lmc,liu2022wireless}. 
  Federated learning is a general machine learning scheme that assumes that the data is distributed amongst different devices/clients and  a central server controls them. 
  Such a setting occurs, for example, in mobile phone applications, where each device has its data, and a (slow) internet connection with the server \cite{konevcny2016federated,mcmahan2017communication-efficient}. 
  That is why, in the most cases the computational bottleneck is in the communication complexity. 
  The goal is to train a global model using local updates communicating as less information as possible.

  Mathematically, our problem is formulated as follows. 
  The target distribution $\pi$ is continuous and is defined on the Euclidean space $\RR^d$. 
  With abuse of notation, we will denote by $\pi$ also the density of the target:
  \begin{equation}\label{eq:density}
    \pi(x) \propto \exp(-F(x)).
  \end{equation}
  Here $F:\RR^d \rightarrow \RR$ is called a potential function. 
  In the general Bayesian setting $F$ is the log-posterior. 
  The latter is assumed to be sum-decomposable in the federated setting, where each component is stored 
  on one of the clients (or nodes/devices):
  \begin{equation*}
    F(x) = \frac{1}{n} \sum_{i=1}^n F_i(x).
  \end{equation*}
  The $i$-th node has access only to its respective score, that is, the gradient $\nabla F_i(x)$.
  Following this scheme, we introduce three sampling algorithms that combine Langevin Monte-Carlo with known 
  federated optimization techniques called EF21 \cite{richtarik2021ef21} and EF21-P \cite{sasha_kaja_EF-21P}: 
  \begin{itemize}
     \vspace{-.3cm}
     \item D-ELF: LMC with dual  compression (\Cref{sec:uplink}); 
     \vspace{-.2cm}
     \item  P-ELF: LMC with primal  compression (\Cref{sec:downlink});
     \vspace{-.2cm}
     \item B-ELF: LMC with bidirectional compression (\Cref{sec:bidir}).
     \vspace{-.3cm}
   \end{itemize} 
  The first algorithm uses only client-to-server (uplink) compression to reduce communication complexity. 
  This assumption was proposed in early federated learning papers, such as \cite{konevcny2016federated}. It was justified by the standard assumption that the uplink communication was more costly  compared to the server-to-client communication.
  However, in a more recent report from Speedtest.net\footnote{\url{https://www.speedtest.net/global-index}} we notice that the difference between uploading and downloading speeds is negligible  \cite{philippenko2020bidirectional}. 
  Therefor, downlink compression is as important as the uplink one.
  The second algorithm replicates the EF21 scheme for the primal space and thus applies compression to the  server-to-client communication. 
  Finally, the third algorithm has both uplink and downlink compression, hence the term bidirectional.
  Bidirectional federated optimization has recently been studied by several authors \cite{liu2020double,philippenko2020bidirectional,sasha_kaja_EF-21P}. 
  However, this setting have not yet been developed and studied for sampling problems. 
  In this work, we analyze the first federated sampling algorithm with bidirectional compression. 

  % This fits perfectly the setting of Bayesian statistics, where we often have access only to the log-posterior which is the product of the prior and the likelihood, while the normalization constant is an intractable integral.

  \subsection{Langevin sampling}

  A common way to solve this problem is based on discretizing a stochastic differential equation (SDE) called Langevin diffusion (LD).
  LD was initially designed to model the movement of particles in an environment with friction \cite{risken1996fokker}. 
  Mathematically, it is written as 
  \begin{equation*}
    \rmd L_t = - \nabla F(L_t)\rmd t + \sqrt{2}\rmd B_t,
      % L_t = L_0  - \int_0^t \nabla F(L_s)\rmd s + \sqrt{2t} Z,
  \end{equation*}
  where $B_t$ is the Brownian motion and $F$ is the potential function from \eqref{eq:density}. 
  The critical property of this SDE is that it has a solution and is ergodic under mild conditions. 
  Moreover, the target $\pi$ is its invariant distribution \cite{bhattacharya1978criteria}.
  Let us now define by $\rho_t$ the density of $L_t$. 
  Then, the evolution of $\rho_t$ is characterized by the Fokker-Planck equation corresponding to LD 
  \cite{pavliotis2014stochastic,risken1996fokker}:
  \begin{equation*}
    \frac{\partial \rho_t(x)}{\partial t} =  \nabla\cdot \brr{F(x) \rho_t(x)} + \Delta \rho_t(x).
  \end{equation*}
  Using the chain-rule in Fokker-Planck equation, one can verify that $\pi$ is indeed the stationary distribution for the Langevin diffusion.

  Langevin Monte-Carlo (LMC) is the Euler-Maruyama discretization of the Langevin diffusion \cite{parisi1981correlation}. 
  That is,
  \begin{equation}\label{eq:lmc}
    x_{k+1} = x_k - \gamma \nabla F(x_k) + \sqrt{2\gamma} Z_k,
  \end{equation}
  where $(Z_k)_k$ is a sequence of i.i.d. standard Gaussians on $\RR^d$ that are independent of previous iterations. 
  When the score function is Lipschitz continuous, and the target satisfies Log-Sobolev inequality (see \Cref{def:LSI}), the distribution of the $K$-th iterate converges to $\pi$ \cite{vempala2019rapid}. See \Cref{sec:related_work} for more context on the LMC.

  \subsection{EF21 and EF21-P}\label{sec:ef}

  The Error Feedback algorithm first appeared in a heuristic manner in the paper by \citet{seide20141}.
  It was proposed as a stabilization mechanism for supervised learning using contractive compressors.
  Later, \citet{alistarh2018convergence,stich2018sparsified} analyzed the method theoretically. 
  Nevertheless, the initial EF has numerous issues:
  \begin{itemize}
     \vspace{-.3cm}
    \item It does not generalize to the distributed setting, which is crucial to federated learning.
     \vspace{-.2cm}
    \item The convergence analysis requires unrealistic assumptions, such as bounds on the gradient norm.
     \vspace{-.3cm}
  \end{itemize}
  See also  Section 2 of \cite{horvath2020better} for more details on the shortcomings of the Error Feedback method.
  The EF21 (Error Feedback 21) algorithm is a modification of the original EF proposed by \citet{richtarik2021ef21}.
  The method follows the standard setting and compresses gradients before communicating them to the server.
  It essentially solves the above issues, and in particular, it is applicable to the distributed setting.
  The method is state of the art in terms of theory and practice amongst error feedback mechanisms \cite{fatkhullin2021ef21}.
  We refer the reader to \Cref{sec:uplink} for the exact definition and mathematical details of the EF21.

  EF21-P is a primal error-feedback method largely inspired from EF21.
  It can also be seen as a reparametrization of the original error feedback method. 
  The method is essentially the analog of EF21 on the primal space.
  Contrary to the dominating approach in federated learning \cite{konevcny2016federated,stich2018sparsified,mishchenko2019distributed,richtarik2021ef21,fatkhullin2021ef21}, it performs compression on iterates of the algorithm rather than their gradients.
  Hence, it reduces the downlink communication complexity. 
  In general, efficient server-to-client compression may play a key role when the model is extremely large \cite{dean2012large,brown2020language}. 
  Furthermore, according to \cite{sasha_kaja_EF-21P}, EF21-P can also be viewed as an iteration perturbation method. 
  These methods are used in various settings in machine learning, including generalization \cite{orvieto2022anticorrelated} and smoothing \cite{duchi2012randomized}. For a precise mathematical definition of the method, see \Cref{sec:downlink}.

  \subsection{Related work}\label{sec:related_work}

  In their seminal paper, \citet{roberts1996exponential} study the convergence properties of the LMC algorithm. 
  They argued that a bias occurs when discretizing the continuous SDE. 
  Thus, Langevin Monte-Carlo generates a homogeneous Markov chain whose stationary distribution differs from the target $\pi$. 
  They solve this issue with a Metropolis-Hastings adjustment step  at each iteration of the LMC, which modifies the chain to have $\pi$ as its stationary distribution.
  The resulting algorithm is called Metropolis Adjusted Langevin Algorithm (MALA), and it was studied by many \cite{roberts1998optimal,roberts2002langevin,xifara2014langevin,dwivedi2018log}.
  
  The bias of the LMC, however, depends on the discretization step $\gamma$. 
  \citet{dalalyan2017theoretical} proved a bound on this error.
  Thus, similar to the analysis of the SGD, controlling the step-size and taking enough iterations, one can make the error of the LMC algorithm smaller than any $\varepsilon$.
  Later, different properties of the LMC were studied by many 
   \cite{durmus2017nonasymptotic,cheng2018sharp,cheng2018convergence,dalalyan2019user,durmus2019high,vempala2019rapid}.

  Looking closely at \eqref{eq:lmc}, we observe its similarity with the gradient descent (GD) algorithm.
  In fact, \eqref{eq:lmc} is an instance of the stochastic gradient descent (SGD) with a Gaussian noise independent of the iterate.
  This similarity has been repetitively exploited in various settings for sampling problems 
  (see e.g. \cite{raginsky2017non,chatterji2018theory,wibisono2019proximal,salim2019stochastic,karagulyan2020penalized}). 
  In particular, a line of research has been initiated on federated  sampling Langevin algorithms, which combine LMC with existing optimization mechanisms: LMC+FedAvg \cite{mcmahan2017communication-efficient,deng2021convergence,plassier2022federated}, LMC+MARINA \cite{gorbunov2021marina, marina-langevin}, LMC+QSGD \cite{alistarh2017qsgd,pmlr-vono22a-fed-lmc}. Our work continues the logic of these papers by adding the error-feedback mechanisms EF21 and EF21-P to the classic LMC algorithm in the federated setting.

  As in the case of optimization, the strong convexity of the potential function plays an important role in the analysis of Langevin Monte-Carlo. 
  Non-convex optimization, however, has long been a central topic in the domain. 
  We refer the reader to \cite{jain2017non} for an overview of non-convex optimization in machine learning.
  In comparison, sampling from non-strongly log-concave distributions is less studied. 
  \citet{cheng2018sharp} studied convergence of the LMC when the potential strongly convex outside a ball.
  \citet{dalalyan2019bounding} and \cite{karagulyan2020penalized} proposed a penalization of the convex potential to make it strongly convex and gave convergence bounds depending on the penalty. 
  The analysis of MALA in the non-convex regime can be found in \cite{mangoubi2019nonconvex}. 
  However, these results either do not cover the general non-convex case or they require some conditions that scale poorly with the dimension. 
  A more efficient approach relies on isoperimetric inequalities. 
  It is known that isoperimetry implies rapid mixture of the continuous stochastic processes \cite{villani2008optimal}.
  Thus, one would assume that this property could be extended to their discretizations. 
  \citet{vempala2019rapid} proved the convergence of the LMC under Log-Sobolev inequality.  
  Later, \citet{marina-langevin} used this as a general scheme for LMC with stochastic gradient estimators in the context of federated Langevin sampling. 
  We simplify their proof and adapt it to our setting.

\subsection{Structure of the paper}

This paper is organized as follows. 
\Cref{sec:framework} describes the mathematical framework of the problem, the notation, definitions, and assumptions.
In Sections \ref{sec:uplink} and \ref{sec:downlink} we present respectively the downlink and uplink compressed Langevin algorithms. That is D-ELF and P-ELF. 
In \Cref{sec:bidir} we introduce our bidirectional federated Langevin algorithm: B-ELF.
The main convergence results are presented in  \Cref{sec:theorems}.
The analysis of all three methods are influenced by  \cite{vempala2019rapid} and \cite{marina-langevin}.  
We simplify their proofs and adapt it to our method; see \Cref{sec:proof-scheme} and \Cref{sec:proof_ef21-lmc}.  
We present the general proof scheme in \Cref{sec:proof-scheme}, while the actual proofs are postponed to the Appendix. We conclude the main part of the paper with \Cref{sec:conclusion}. 

\section{Problem setup}\label{sec:framework}

We denote by $\RR^d$ the $d$-dimensional Euclidean space endowed with its usual scalar product and 
${\ell}_2$-norm defined by 
$\langle \cdot, \cdot\rangle$ and $\norm{\cdot}$.
The gradient of the function $H$ and its Hessian evaluated at the point $x \in \RR^d$ is denoted by $\nabla H (x)$ and $\nabla^2 H(x)$, respectively.
As mentioned previously, we will repeatedly use the same notation for probability distributions and their corresponding densities.
For the asymptotic complexity of the algorithms we will use the $\cO$ and $\tilde{\cO}$ notations. We say that $f(t) = \cO(g(t))$ when $t \rightarrow +\infty$, if $f(t) \leq Mg(t)$, when $t$ is large enough. 
Similarly, $f(t) = \tilde{\cO}(g(t))$, if $f(t)\log(t) = \cO(g(t))$.

\subsection{Mathematical framework}

The vast majority of optimization and sampling literature relies on the $L$-smoothness assumption. 
\begin{definition}[$L$-smoothness]
  We say that a function is $L$-smooth, if 
  \begin{equation*}
    F(y) \leq F(x) + \inner{\nabla F(x)}{y-x} + \frac{L\normsq{x-y}}{2}. 
  \end{equation*}
\end{definition}
This, in particular, yields that the gradient is $L$-Lipschitz continuous function:
\begin{equation*}
   \norm{ \nabla F(y) - {\nabla F(x)} } \leq {L\norm{x-y}}. 
\end{equation*}

EF21 and EF21-P rely on contractive compressors to reduce the communication complexity. 
\begin{definition}[Contractive compressor]
  \label{def:compression}
  A stochastic mapping $\mathcal{Q}: \mathbb{R}^{d} \rightarrow \mathbb{R}^{d}$ is a contractive compression operator with a coefficient $\alpha \in (0,1]$ if for any $x \in \mathbb{R}^{d}$,
  \begin{equation*}
    % \label{eq:comp-var}
   \Exp{\|\mathcal{Q}(x)-x\|^{2}} 
    \leq (1 - \alpha)\|x\|^{2}.
  \end{equation*}
  We denote it shortly as $\cQ \in \BB(\alpha)$.
\end{definition}
Here, we notice that we do not require unbiasedness. 
In many federated learning algorithms, unbiased compressors with bounded variance are used (see e.g. \cite{konevcny2016federated,alistarh2017qsgd,mishchenko2019distributed,gorbunov2021marina}). 
Unbiased compressors are defined as (possibly stochastic) mappings such that $\EE[\cQ(x)] = x$ and $\EE[\normsq{\cQ(x) - x}] \leq \omega \normsq{x}$. Then, simple computation shows that $\frac{1}{\omega+1}\cQ$ is a $\frac{1}{\omega+1}$-contractive compressor.
However, the class of contractive compressors is strictly larger.
Indeed. Let us look at the Top-$k$ compressor \cite{alistarh2017qsgd}.
This compressor returns only the $k$ coordinates with the largest absolute values of the input vector.
It is obvious that Top-$k$ cannot be represented with unbiased compressors, as it is deterministic.
This, concludes the argument.

Our analysis relies on the interpretation of sampling as an optimization problem over the space of measures. In order to reformulate our problem, let us first recall the definition of the 
 Kullback-Leibler divergence.
\begin{definition}[Kullback-Leibler divergence]
  The Kullback-Leibler divergence between two probability measures $\nu$ and $\pi$ is defined as
   \begin{equation*}
    \KL{\nu} = 
    \begin{cases}
       \int_{\RR^d}  \log\left(\frac{\nu(x)}{\pi(x)}\right)\nu(x)\rmd x, \text{ if } \nu \ll \pi;\\
       +\infty, \text{ otherwise}.
    \end{cases}
  \end{equation*}
\end{definition}

We aim to construct approximate samples from $\pi$ with $\varepsilon$ accuracy. 
That is to sample from some other distribution $\nu$ such that $\KL{\nu} < \varepsilon$.
Alternatively, it means that we want to minimize the functional:
\begin{equation*}
  \min_{\nu \in \cP(\RR^d)} \KL{\nu}.
\end{equation*}
Indeed, the minimum of this functional is equal to zero and  is attained only when $\nu = \pi$.
Recall now the classical problem of optimization, that is minimizing a $H: \RR^d \rightarrow \RR$. \citet{polyak1963gradient} and \citet{lojasiewicz1963topological} independently proposed an inequality, which is weaker than strong convexity, but it nevertheless implies linear convergence of the gradient descent.
It is known under the joint name of Polyak-{\L}ojasiewicz inequality:
\begin{equation*}
  H(x) - \min_x H(x) \leq  \frac{1}{\mu} \normsq{\nabla H(x)},
\end{equation*}
assuming the objective has a minimum.
See \cite{karimi2016linear,khaled2020better} for more details on the PL inequality, as well as its comparison with other similar conditions for non-convex optimization.
In the problem of sampling, the objective functional is defined on the space of measures $\cP(\RR^d)$.
One can define the usual notions of differentiability and convexity on this space using the Wasserstein distance \cite{ambrosio2008gradient}.
Then, the Langevin Monte-Carlo algorithm becomes a first order minimization method for the KL divergence \cite{wibisono2018sampling}.
Furthermore, Fisher information takes the role of the square norm of the gradient.
\begin{definition}[Fisher information]
 The Fisher information of probability measures $\nu$ and $\pi$ is denoted by $\FS{\nu}$ and it is defined as below:
  \begin{equation*}
    \FS{\nu}:=
   \begin{cases}
       \int_{\RR^d}\normsq{\nabla \log\brr{\frac{\nu}{\pi}}}\nu(x)\rmd x, \text{ if } \nu \ll \pi;\\
       +\infty, \text{ otherwise}.
    \end{cases}
  \end{equation*}
\end{definition}
Since the minimum of our functional is equal to zero, the Log-Sobolev inequality (LSI) becomes the analog of PL inequality.
\begin{definition}[Log-Sobolev inequality]
    \label{def:LSI}
  We say that $\pi$ satisfies the Log-Sobolev inequality (LSI) with parameter $\mu$, if
  for every probability measure $\nu$ we have
  \begin{equation*}
    \KL{\nu}\leq \frac{1}{2\mu}\FS{\nu}.
  \end{equation*} 
\end{definition}

\citet{bakry1985diffusions} have shown that strongly log-concave distributions satisfy LSI.
Furthermore, from Holley-Stroock's theorem we know that sufficiently small perturbations of strongly concave distributions still satisfy LSI \cite{holley1986logarithmic}. 
The latter distributions can be non log-concave, which means that we deal with a strictly larger class of probability measures using LSI.  

Analyzing the sampling problems as an optimization problem on the Wasserstein space has been strongly influenced by the seminal paper of \citet{jordan1998variational}. It has later been developed in subsequent work; see e.g. \cite{wibisono2018sampling,durmus2018analysis}. 

We use Log-Sobolev inequality to derive bounds on the convergence error in KL divergence. 
This bounds can also be extended to other probability distance metrics.
\begin{definition}[TV]
    \label{def:TV}
  Let $\nu_1$ and $\nu_2$ be two measures defined on $(\RR^d,\cB)$, where $\cB$ is the class of all Borell sets.
  Then, the total variation distance of $\nu_1$ and $\nu_2$ is defined as
  \begin{equation*}
    TV(\nu_1,\nu_2) = \sup_{A \in \cB} \big| \nu_1(A) - \nu_2(A) \big|.
  \end{equation*} 
\end{definition}

The relation of TV and KL is established with Pinsker's inequality:
\begin{equation}
  TV(\nu_1,\nu_2) \leq \sqrt{\frac{1}{2} {H_{\nu_2}\left( \nu_1 \right)}}.
\end{equation}
Thus, the convergence in KL divergence implies convergence in TV.
Similar result is true for the Wasserstein-2 distance.
\begin{definition}[Wasserstein-2]
    \label{def:W2}
   Let  $\nu_1,\nu_2 \in \mathcal{P}_{2} (\RR^d)$. That is, their second moments are finite. 
  The Wasserstein-$2$ distance between two probability measures is defined as 
  \begin{equation*}
    W_2(\nu_1,\nu_2) := \inf _{\eta \in {\Gamma}(\nu_1, \nu_2)} \left[\int\|x-y\|^{2} 
    \eta(\rmd x, \rmd y)\right]^{1/2},
  \end{equation*}
  $\Gamma(\nu_1,\nu_2)$ is the set of all 
  joint distributions defined on $\RR^d \times \RR^d$ having $\nu_1$ and $\nu_2$ as its marginals (also known as couplings).
\end{definition}
It is known that LSI implies Talagrand's inequality. The latter bounds the $W_2$ distance with KL divergence:
\begin{equation}
    W_p(\nu,\pi)\leq \sqrt{\frac{2\KL{\nu}}{\mu}}, \quad \forall \nu \in\cP_2(\RR^d).
\end{equation}
Again, from the convergence in KL we can deduce convergence in $W_2$.

\begin{algorithm}[h!]
  \caption{D-ELF}\label{alg:ef21_langevin}
  \begin{algorithmic}[1]
    \STATE {\bfseries Input:} Starting point $x_0 \sim\rho_0 $, $g^i_k = g_k = \nabla F(x_0)$, step-size $h$, number of iterations $K$
    \FOR {$k=0,1,2,\ldots,K-1$}
    
    \STATE \underline{The server:} 
    \STATE \hspace{\algorithmicindent} draws a Gaussian vector $Z_{k}\sim\mathcal{N}(0,I_{d})$;
    \STATE \hspace{\algorithmicindent} computes $x_{k+1}=x_k-\gamma g_{k} \underline{\underline{+ \sqrt{2\gamma}Z_{k}}}$;
    \STATE  \hspace{\algorithmicindent}  broadcasts $x_{k+1}$;
    \STATE \underline{The devices in parallel:}  
    \STATE  \hspace{\algorithmicindent}  compute $\cQ^{\rm D}(\nabla F_i(x_{k+1}) - g^i_{k})$.
    \STATE \hspace{\algorithmicindent} compute $g^i_{k+1} = g^i_k + \cQ^{\rm D}(\nabla F_i(x_{k+1}) - g^i_k)$;
    \STATE \hspace{\algorithmicindent}  broadcast  $\cQ^{\rm D}(\nabla F_i(x_{k+1}) - g^i_k)$;  
    \STATE \underline{The server:} 
    \STATE   computes $ g_{k+1}  =  g_k + \frac{1}{n} \sum_{i=1}^{n} 
    \cQ^{\rm D}( \nabla F_i(x_{k+1}) - g^i_{k})$.
    \ENDFOR
    \STATE {\bfseries Return:} $x_K$
  \end{algorithmic}
\end{algorithm}
\section{The ELF algorithms}\label{sec:ef}

  In this section we present two federated Langevin Monte-Carlo algorithms. 
  We combine EF21 and EF21-P with LMC and provide a unified analysis for both methods. 
  At each iteration of the modified algorithms, we replace the gradient term $\nabla F(x_k)$ by $g_k$ where the latter is the gradient estimator from one of the corresponding error feedback method. 
  In other words, we take the original algorithms and insert an independent Gaussian noise at each update. See \Cref{alg:ef21_langevin} and \Cref{alg:p-elf} for details. 
  The difference between the optimization and sampling methods is underlined with a wave in the pseudocode.

  \subsection{Dual compression: D-ELF}\label{sec:uplink}

  The gradient estimator $g_k$ of the dual method is defined as the average of the vectors $g^i_k$,
  where each $g^i_k$ is computed on the $i$-th node and estimates the gradients $\nabla F_i(x_k)$. 
  The key component of this estimator is the contractive compression operator $\cQ^{\rm D} \in \BB(\alpha^{\rm D}) $. 
  At the zeroth iteration, $g_0 = \nabla F(x_0)$. 
  Then at iteration $k$, the server computes the new iterate $x_{k+1}=x_k-\gamma g_{k} {+ \sqrt{2\gamma}Z_{k}}$ and broadcasts it parallelly to all the nodes. 
  Each node  updates $g^i_k$ with the formula:
  \begin{equation*}
    g^i_{k+1} = g^i_k + {\cQ^{\rm D}(\nabla F_i (x_{k+1}) - g^i_k)},
  \end{equation*}
  and broadcasts the compressed term to the server.
  The server aggregates the received information and computes the estimator of $\nabla F_i (x_{k+1})$:
  \begin{equation*}
    g_{k+1} = g_k + \frac{1}{n}\sum_{i=1}^{n}\cQ^{\rm D}(\nabla F_i (x_{k+1}) - g^i_k).
  \end{equation*}
  For the pseudocode of the D-ELF, please refer to \Cref{alg:ef21_langevin}.

\subsection{Primal compression: P-ELF}\label{sec:downlink}

The construction of the P-ELF algorithm is similar to the D-ELF. 
In particular, we take the EF21-P algorithm by \citet{sasha_kaja_EF-21P} and add only the independent Gaussian term. See \Cref{alg:p-elf} for the complete definition. 
To better understand the comparison  of  the D-ELF and the P-ELF let us look at the simple one-node setting of the latter:
{
\begin{equation}
  \begin{cases}
    w_0 := \cQ^{\rm P}(x_0)\\
    w_{k+1} = w_k + \cQ^{\rm P}(x_{k+1} - w_k)\\
    x_{k+1} = x_k - \gamma \nabla F(w_{k}).\\
  \end{cases}
\end{equation}
Here, $x_0 \sim \rho_0$ is a random starting point and $(Z_k)_k$ is a sequence of i.i.d. 
standard Gaussians on $\RR^d$.
The auxiliary sequence $w_k$ is meant to estimate to the iterate $x_k$. 
We then use its gradient as the minimizing direction. 
The important difference with the EF21 is that we apply the compressor $\cQ^{\rm P}$ on the term $x_{k+1} - w_k$, instead of the gradient and its estimator. Hence, the letter "P"-primal in the name of the algorithm. 

\begin{algorithm}[h!]
  \caption{P-ELF}\label{alg:p-elf}
  \begin{algorithmic}[1]
    \STATE {\bfseries Input:} Starting point $x_0 = w_0 \sim\rho_0 $, step-size $h$, number of iterations $K$
    \FOR {$k=0,1,2,\cdots,K-1$}
    \STATE \underline{The server:} 
    \STATE \hspace{\algorithmicindent} draws a Gaussian vector $Z_{k}\sim\mathcal{N}(0,I_{d})$;
    \STATE \hspace{\algorithmicindent}  
    computes $\nabla F(w_{k})  = \frac{1}{n}\sum_{i=1}^{n} \nabla F_i(w_{k})$;
    \STATE \hspace{\algorithmicindent} computes $x_{k+1}=x_k-\gamma \nabla F(w_{k}) \underline{\underline{ + \sqrt{2\gamma}Z_{k}}}$;
    \STATE  \hspace{\algorithmicindent}  computes $\cQ^{\rm P}(x_{k+1}-w_k)$;
    \STATE  \hspace{\algorithmicindent} computes $w_{k+1} = w_k + \cQ^{\rm P}(x_{k+1}-w_k)$;
    \STATE  \hspace{\algorithmicindent}  broadcasts in parallel  $\cQ^{\rm P}(x_{k+1}-w_k)$.
    \STATE \underline{The devices in parallel:}  
    \STATE \hspace{\algorithmicindent} compute $w_{k+1} = w_k + \cQ^{\rm P}(x_{k+1}-w_k)$;
    \STATE \hspace{\algorithmicindent}  compute  $\nabla F_i(w_{k+1})$;
    \STATE \hspace{\algorithmicindent}  broadcast  $\nabla F_i(w_{k+1})$;
    \ENDFOR
    \STATE {\bfseries Return:} $x_K$
  \end{algorithmic}
\end{algorithm}

\subsection{Bidirectional compression: B-ELF}\label{sec:bidir}

This section focuses on the bidirectional setting. 
We propose the B-ELF algorithm. 
The algorithm uses EF21 for the uplink and EF21-P for the downlink compression.
The details are presented in \Cref{alg:bidir_langevin}.

\begin{algorithm}[h!]
  \caption{B-ELF}\label{alg:bidir_langevin}
  \begin{algorithmic}[1]
    \STATE {\bfseries Input:} Starting point $x_0 = w_0 \sim\rho_0 $, step-size $h$, number of iterations $K$,
    $g_0 = \nabla f(x_0)$, $g^i_0 = \nabla f_i(x_0)$.
    \FOR {$k=0,1,2,\cdots,K-1$}
    \STATE \underline{The server:} 
    \STATE \hspace{\algorithmicindent} draws a Gaussian vector $Z_{k}\sim\mathcal{N}(0,I_{d})$;
    \STATE \hspace{\algorithmicindent} computes $x_{k+1}=x_k- \gamma g_{k} + \sqrt{2\gamma}Z_{k}$;
    \STATE  \hspace{\algorithmicindent}  computes $v_k := \cQ^{\rm P}(x_{k+1}-w_k)$;
    \STATE  \hspace{\algorithmicindent} computes $w_{k+1} = w_k + v_k$;
    \STATE  \hspace{\algorithmicindent}  broadcasts $v_k$ in parallel to the devices;
    
    \STATE \underline{The device $i$ (in parallel for all $i=1,\ldots,n$):} 
    \STATE \hspace{\algorithmicindent} computes $w_{k+1} = w_k + v_k$;
    \STATE  \hspace{\algorithmicindent}  computes $h^i_{k+1} = \cQ^{\rm D}(\nabla F_i(w_{k+1}) - g^i_{k})$;
    \STATE \hspace{\algorithmicindent} computes $g^i_{k+1} = g^i_k + h^i_{k+1} $;
    \STATE \hspace{\algorithmicindent}  broadcasts  $h_i^{k+1}$;
    \STATE \underline{The server:} 
    \STATE \hspace{\algorithmicindent}  computes $g_{k+1} =  g_k + \frac{1}{n}\sum_{i=1}^{n} h^i_{k+1} $;
  
    \ENDFOR
    \STATE {\bfseries Return:} $x_K$
  \end{algorithmic}
\end{algorithm}

\section{Convergence of the methods}\label{sec:theorems}

\subsection{A unified analysis of D-ELF and P-ELF}\label{sec:delf-analysis}

  The key component of the analysis of both methods is defining proper a Lyapunov-type function.
  For the D-ELF algorithm we define by $\bG^{\rm D}_k$ the average squared estimation error of the vectors $g^i_k$:
  \begin{equation}\label{eq:gk-ef21}
    \bG^{\rm D}_k := \frac{1}{n} \sum_{i}^{n} \EE\brs{\norm{g^i_k - \nabla F_i(x_k)}^2}.
  \end{equation}
  As we will later in \Cref{sec:proof-scheme}, this quantity arises in the proof of the convergence rates. 
  Important property of the sequence $\bG_k$ is the following recurrent identity.
  \begin{proposition}\label{prop:uplink}
    Let $x_k$ be the iterates of the D-ELF, $g^i_k$ be the EF21 estimators and $\bG^{\rm D}_k$ be defined as \eqref{eq:gk-ef21}.
    Then the following recurrent  inequality is true:
    \begin{equation}\label{eq:prop-dual}
      \bG^{\rm D}_{k+1} \leq (1 - p)\bG^{\rm D}_k + (1 - p)\beta_{\rm D}  \EE\brs{\norm{x_{k + 1} - x_k}^2},
    \end{equation}
    where  
    \begin{equation*}
      \begin{aligned}
        p &:= 1 - (1-\alpha_{\rm D})(1+s_{\rm D}) >0 \\ 
        \bar{L} &:= \frac{1}{n}\sum_{i=1}^{n} L_i^2 \quad \text{and} \quad
        \beta_{\rm D} := \frac{1+s_{\rm D}^{-1}}{1 + s_{\rm D}}\bar{L},
      \end{aligned}
    \end{equation*} 
    for some $s_{\rm D}>0$.    
  \end{proposition}
  The Lyapunov term associated to the P-ELF is a simple upper bound on $\bG^{\rm D}$. 
  We denote it by $\bG^{\rm P}_k$ and define with the formula below:
    \begin{equation}\label{eq:tk-def}
    \begin{aligned}
      \bG^{\rm P}_k := \bar{L}\Exp{\norm{w_k - x_k}^2},    
      \hspace{.2cm} \text{where} \hspace{.2cm}
      \bar{L} := \frac{1}{n}\sum_{i=1}^{n} L_i^2.
    \end{aligned}
    \end{equation}
  Indeed, $\bG^{\rm D}_k \leq \bG^{\rm P}_k$ due to $L_i$ smoothness of each component function $F_i$.
  See \eqref{eq:gd-gp} in \Cref{sec:proof_ef21-lmc} for the proof.
  The following proposition proves a recurrent identity similar to \eqref{eq:prop-dual}.
  \begin{proposition}\label{prop:downlink}
    Let $x_k$ and $w_k$ be defined as in P-ELF algorithm and $\bG^{\rm P}$ be its Lyapunov term.
    Then the following recurrent  inequality is true:
    \begin{equation*}
      \bG^{\rm P}_{k+1} 
      \leq (1 - p)\bG^{\rm P}_k + (1 - p)\beta_{\rm P} \EE\brs{\norm{x_{k + 1} - x_k}^2},
    \end{equation*}
    where  
    \begin{equation}
      \begin{aligned}
        p &:= 1 - (1-\alpha_{\rm P})(1+s_{\rm P}) >0, \\ 
        \beta_{\rm P} &:= \frac{1+s_{\rm P}^{-1}}{1 + s_{\rm P}}\bar{L},
      \end{aligned}
    \end{equation} 
    for some $s_{\rm P}>0$.
    
  \end{proposition}
The next theorem gives a unified bound for both D-ELF and P-ELF. 
For the sake of space we use a general notation M-ELF, where $\text{ M} \in \brc{\text{D,P}}$. 
This means, for example, that the M-ELF refers to the D-ELF when $\text{M} = \text{D}$.
\begin{theorem}
  \label{thm:ef21-lmc}
  Assume that LSI holds with constant $\mu > 0$ and let $x_k$ be the iterates  of the M-ELF algorithm, where $\text{M} \in \brc{\text{D,P}}$.   We denote by  $\rho_k := \cL({x_k})$ for every $k \in \NN$. 
  If 
  \begin{equation*}
    0<\gamma\leq \min \left\{\frac{1}{14}\sqrt{\frac{p}{ (1+\beta_{\rm M})}},\frac{p}{6\mu}, \frac{1}{2\sqrt{2}L}\right\},
  \end{equation*}
  then the following is true for the KL error of the M-ELF algorithm:
  \begin{equation*} 
     \KL{\rho_{K}}\leq e^{-\mu K \gamma}\Psi+\frac{\tau}{\mu},
  \end{equation*}
  where 
   { 
   \begin{align*}
      p &:= 1 - (1-\alpha_{\rm M})(1+s_{\rm M}) >0,       \\ 
      \Psi &=\KL{\rho_0}+\frac{1-e^{-\mu \gamma}}{\mu}C\bG^{\rm M}_{0}, \\
      \tau &=\left(2{L^2}+C(1-p) \beta_{\rm M}\right)\left(16 \gamma^2d+4d\gamma\right),\\
      C &=\frac{8{L^2}\gamma^2+2}{e^{-\mu \gamma}-(1-p)\left(4\gamma^2\beta_{\rm M}+1\right)}.
    \end{align*}}
\end{theorem}
We refer the reader to \Cref{sec:proof_ef21-lmc} for the proof of the theorem.
The right-hand side consists of two terms. 
The first term corresponds to the convergence error, while the second term is the bias that comes from the discretization.
To make the error small, one would first need to choose $\gamma$ small enough so that $\tau/\mu < \varepsilon$. 
Then, the number of iterations are chosen to be of order $\tilde{\cO}(\nicefrac{1}{\mu\gamma})$.
See \Cref{sec:discussion} for more on the complexity of the D-ELF and P-ELF.

\begin{table}[t]
  \caption{In this table we compare error-feedback methods in optimization and sampling. 
  The rates are computed in the case when $\alpha_{\rm D} = \alpha_{\rm P} = \alpha$. }
  \label{table1}
  % \vskip 0.15in
  \begin{center}
    \begin{scriptsize}
    \begin{sc}
      \begin{tabular}{lccl}
      \toprule
        Method &  \hspace{-.3cm}  Assumption  & \hspace{-.3cm} Complexity &\hspace{-.3cm}  Reference \\ 
        \midrule
        GD    & \hspace{-.3cm} $\mu$-s.c. & \hspace{-.3cm}  $\tilde{\cO}\brr{\frac{dL}{\mu\varepsilon}}$ & \hspace{-.3cm} \cite{nesterov2013introductory}\\ \vspace{.15cm}
        EF21 &  \hspace{-.3cm} $\mu$-s.c.& \hspace{-.3cm}  $\tilde{\cO}\brr{\frac{L}{\alpha\mu\varepsilon}}$ &\hspace{-.3cm}  \cite{richtarik2021ef21}\\\vspace{.15cm}
        EF21-P  &  \hspace{-.3cm} $\mu$-s.c.& \hspace{-.3cm}  $\tilde{\cO}\brr{\frac{L}{\alpha\mu\varepsilon}}$ & \hspace{-.3cm} \cite{sasha_kaja_EF-21P} \\
        \midrule
        \vspace{.15cm}
        LMC    & \hspace{-.3cm}   $\mu$-LSI  & \hspace{-.3cm}    $\tilde{\cO}\brr{\frac{L^2 d}{\mu^2\varepsilon}}$ &\hspace{-.3cm}  \cite{vempala2019rapid}     \\\vspace{.15cm}
        D-ELF    &  \hspace{-.3cm}   $\mu$-LSI  &  \hspace{-.3cm} $\tilde{\cO}\brr{\frac{\bar{L}d }{\alpha^2\mu^2\varepsilon}}$  &\hspace{-.3cm}  \Cref{corr:delf-conv}\\\vspace{.15cm}
        P-ELF    &  \hspace{-.3cm}   $\mu$-LSI  &  \hspace{-.3cm} $\tilde{\cO}\brr{\frac{\bar{L}d }{\alpha^2\mu^2\varepsilon}}$  &\hspace{-.3cm}  \Cref{corr:delf-conv}\\\vspace{.15cm}
        B-ELF     &  \hspace{-.3cm}  $\mu$-LSI  & \hspace{-.3cm}   $\tilde{\cO}\brr{\frac{\bar{L} d}{\alpha^4\mu^2\varepsilon}}$& \hspace{-.3cm} \Cref{corr:belf-conv}  \\
      \bottomrule
      \end{tabular}
    \end{sc}
    \end{scriptsize}
  \end{center}
  \vskip -0.1in
\end{table}

\subsection{Convergence analysis of the B-ELF}\label{sec:belf-analysis}

The Lyapunov term for the B-ELF algorithm is the as for the D-ELF, that is  $\bG_k^{\rm D}$.
However, the recurrent identity of \Cref{prop:uplink} is not valid in this case.
Instead, another bound is true which  includes the term $\bG^{\rm P}_k$.  
The latter arises because of the downlink compression.
We present \textit{informally} the new recurrent inequality.
We refer the reader to \Cref{prop:bidir} in the Appendix for the complete statement.
\begin{proposition}[Informal]\label{prop:bidir_incom}
  If $x_{k}$ are the iterations of \Cref{alg:bidir_langevin}, $\bG^{\rm D}_k$ and $\bG^{\rm P}_k$ are defined as in \eqref{eq:gk-ef21} and \eqref{eq:tk-def}, then  
  \begin{equation*}
      \bG^{\rm D}_{k+1} \leq \lambda_1 \bG^{\rm D}_k + \lambda_2 \EE\brs{\norm{ x_{k}-  x_{k+1}}^2} + \lambda_3 \bG^{\rm P}_k,
  \end{equation*}
  where  $\lambda_1,\lambda_2$ and $\lambda_3$ are positive numbers. 
\end{proposition}

\begin{theorem}\label{thm:bidir}
  Assume that LSI holds with constant $\mu > 0$ and let $x_k$ be the iterates  of the B-ELF algorithm.  
  We denote by  $\rho_k := \cL({x_k})$ for every $k \in \NN$. 
  Let the step-size the following condition:  
  \begin{equation*}
    \begin{aligned}
      \gamma 
      &\leq  \min\brc{\frac{\alpha_{\rm D}}{4\mu},\frac{\alpha_{\rm P}}{4\mu}, \frac{\alpha_{\rm D}\alpha_{\rm P}}{495\sqrt{\brr{1-\frac{\alpha_{\rm D}}{2} }\brr{1-\frac{\alpha_{\rm P}}{2}} \bar{L}}}}.
    \end{aligned}
  \end{equation*}
  Then, for every $K \in \NN$,
  \begin{equation*}
    \begin{aligned}
      {\KL{\nu_K}}  
      & \leq e^{-\mu \gamma K} \brs{ \KL{\rho_0} 
      +\frac{1 }{\mu} \brr{  C \bG^{\rm D}_0 +   D\bG^{\rm P}_0} } + \frac{\tau }{\mu}, 
    \end{aligned}
  \end{equation*}
  where $C,D>0$ are constants depending on the parameters of the algorithm and 
  \begin{equation*}
    \begin{aligned}
      C &= \frac{2.125}{e^{-\mu \gamma} -\lambda_1}, \\
      D &=\frac{ C\lambda_3}{ {e^{-\mu \gamma}  - (1-\alpha_{\rm P}) (1+w) }},\\
      \tau &= \brr{2L^2 + \frac{5C\lambda_2}{\alpha_{\rm P}}}\brr{16\gamma^2{dL}+ 4d\gamma}.
    \end{aligned}
  \end{equation*}
\end{theorem}
The exact definitions of the undefined constants are written in the proof of the theorem, which is postponed to \Cref{sec:proof_bidir}. 
\subsection{Discussion on the communication complexity}\label{sec:discussion}

Doing the computations as mentioned at the end of \Cref{sec:delf-analysis}, we can deduce the following.
\begin{corollary} \label{corr:delf-conv}
  Under the assumptions of \Cref{thm:ef21-lmc} and
  $\gamma=\mathcal{O}\left(\frac{\mu p\varepsilon}{\beta_{\rm M} d}\right)$,  $K=\Omega\left(\frac{(1+\beta_{\rm M}) d}{\mu^2 p\varepsilon}\log \left(\frac{\Psi}{\varepsilon}\right)\right)$, the primal and dual ELF algorithms satisfy 
  $\KL{\rho_{K}} \leq \varepsilon$.
  \end{corollary}  
Similarly, for the bidirectional ELF we have the below.
\begin{corollary}\label{corr:belf-conv}
  If $\alpha_{\rm P} = \alpha_{\rm D} < 1/2$, under the conditions of \Cref{thm:bidir}, the iteration complexity for the B-ELF is 
  $\tilde{\cO}(\nicefrac{d \bar{L}}{\alpha^4 \mu^2\varepsilon})$.
\end{corollary}
The proof of \Cref{corr:belf-conv} can be  found  in \Cref{sec:proof-corr-belf}.
In the case when $\alpha$ is of order $O(1)$, we recover the rate of the LMC algorithm for all three algorithms.
Such setting occurs for the scaled unbiased compressors, such as the natural compressor $\cQ^{\rm nat}$ \cite{horvath2019natural}.
We know that  $\frac{8}{9}\cQ^{\rm nat}$ is contractive with coefficient $\frac{8}{9}$. 
For other type of compressors our analysis may fall behind the usual LMC. 
The communication complexity of the LMC is of order  $\tilde{\cO}(d^2/\varepsilon)$, as at each iteration we communicate $O(d)$ floats.
According to our theoretical findings, the iteration complexity of B-ELF is $\tilde{\cO}(d^5/\varepsilon)$. 
Since Top-1 communicates only one coordinate, the communication complexity is also  $\tilde{\cO}(d^5/\varepsilon)$. 
Nevertheless, LMC is faster in theory. Similar arguments can be derived also for the B-ELF.

However, as it is observed in practice, error feedback based optimization methods outperform gradient descent, despite the large theoretical complexity \cite{richtarik2021ef21}. 
We expect that suitable experiments will confirm that ELF is better than the usual LMC.

\section{General scheme of the proofs}\label{sec:proof-scheme}

For all three algorithms the update of the LMC iteration is a stochastic estimator of the 
gradient $\nabla F(x_k)$.  
Generally, it depends on $x_{k}$ and  $\xi_k$, where $\xi_k$ is a sequence of i.i.d. random variables defined  on some probability space $(\Xi,\cF,\cP)$. 
The sequence $\xi_k$ comprises the randomness that arises at each step of the particular algorithm and it is independent of $x_k$. 
% Thus, alternatively $g_k$ can be defined as $f_\xi (x_k)$ for some function
% $f_{\cdot}(\cdot): \RR^d\times\Xi\to\RR^d$, where $\left(\Xi,\rho,\mathcal{F}\right)$ is a probability space.
In order to prove convergence in KL divergence, we use the interpolation method proposed in 
\cite{vempala2019rapid}.
The method is based on the Fokker-Planck equation of the Langevin diffusion. 
We state a lemma for general LMC algorithms with stochastic drift terms. 
In particular, all our algorithms can be generally written as
\begin{equation}\label{eq:stoch-lmc}
  x_{k+1} = x_k - \gamma f_{\xi_k}(x_k) + \sqrt{2\gamma} Z_k,
\end{equation}
where $\xi_k$ are i.i.d. random variables defined on some probability space $(\Xi,\cF,\cP)$. 
On the other hand, each step can be seen as a realization of a Langevin diffusion with a constant drift term
$ f_{\xi_k}(x_k)$:
\begin{equation}\label{eq:ld-stoch}
  \rmd y_t =  - f_{\xi_k}(x_k)\rmd t + \sqrt{2}\rmd B_t, 
\end{equation}
 with  $y_0 = x_k$ { and } $t \in [0,\gamma]$. 
Indeed,
\begin{equation*}
  \begin{aligned}
    y_\gamma 
    &= y_0 - \int_0^\gamma f_{\xi_k}(y_0)\rmd t + \sqrt{2}(B_\gamma - B_0)\\
    &= x_k -  \gamma f_{\xi_k}(x_k) + \sqrt{2\gamma}Z_1 = x_{k+1}.
  \end{aligned}
\end{equation*}
The interpolation method is based on analyzing the Fokker-Planck equation of this diffusion.
In particular, we will upper bound the time derivative of $\KL{\rho_t}$:
\begin{equation}
    \begin{aligned}
      \dif{\KL{\rho_t}}
      &=\int_{\RR^d} \frac{\partial \rho_{t}(z)}{\partial t} \log\brr{\frac{\rho_t}{\pi}}(z)\rmd z.
    \end{aligned}
  \end{equation}
Here, the first term of the product under the integral can be computed using the abovementioned Fokker-Planck equation. 
The following lemma is the cornerstone of our analysis. 
\begin{lemma}\label{lem:fp-transformation}
If $y_t$ is the solution of the diffusion \eqref{eq:ld-stoch} and $\rho_t = \cL(y_t)$, then for 
every $t \in [0,\gamma]$,
  \begin{equation}\label{eq:ef21-lmc-tytytyti}
    \dif{\KL{\rho_t}} \leq -\frac{3}{4}\FS{\rho_t}+\Exp{\normsq{f_{\xi_k}(y_0)-\nabla F(y_t)}}.
  \end{equation}
\end{lemma}
The bound \eqref{eq:ef21-lmc-tytytyti} was initially derived by \citet{vempala2019rapid} for the standard Langevin Monte-Carlo. 
Its current stochastic form was later proved in \cite{marina-langevin} for MARINA Langevin algorithm.
The proof is postponed to \Cref{sec:lem-fp-trans}.

\Cref{lem:fp-transformation} is valid for all our algorithms. 
We then insert the value of the gradient estimator for each method and bound the last term by 
$\bG^{\rm D}_k$. Using the recurrent properties of the Lyapunov terms and replacing 
Fisher information term by Kullback-Leibler divergence with LSI
inequality we conclude the proof.

\section{Conclusion}\label{sec:conclusion}

In this paper we proposed three error feedback based federated Langevin algorithms with dual, primal and bidirectional compression.
The first two are analyzed with one theorem and have similar theoretical performance. 
The third algorithm uses bidirectional compression which is slower due to the fact that EF21 and EF21-P do not couple. 
To the best of our knowledge, this is the first study of the federated sampling algorithms with bidirectional compression.
Our theoretical findings show that the communication complexity of this algorithm is worse than the one for the standard LMC. 
Nonetheless, in practice error feedback based methods outperform other compression methods \cite{fatkhullin2021ef21}. 
 We believe that this phenomenon shall also transfer to the sampling case.

\subsection{Future work}

An immediate continuation of our paper would be to conduct an experimental analysis of the ELF algorithms with other federated sampling techniques on real data. 
One would expect that one would observe the same behavior as in the optimization case. 
That is, in practice the ELF algorithms outperform the other methods, despite the theoretical analysis.

Another possible direction is the theoretical analysis of the Langevin algorithm combined with EF21-P+DIANA. 
The latter is a bidirectional federated optimization algorithm that uses DIANA gradient estimator for the uplink compression instead of EF21. This method matches the performance of the GD due to the coupling of two methods \cite{sasha_kaja_EF-21P}.

Finally, there are yet many important algorithms of optimization that are relevant to our setting. 
Adaptation of these methods to the sampling setting can lead to fruitful results.

%%%%%%%%%%%%%%%%%%%%%%%%%
%%%%%%%%%%%%%%%%%%%%%%%%%
%%%%%%%%%%%%%%%%%%%%%%%%%

\section*{Acknowledgements}

The work of P.R. was supported by the KAUST Baseline Research Funding (BRF) scheme and also partially supported by the SDAIA-KAUST Center of Excellence in Data Science and Artificial Intelligence.

\clearpage

\bibliographystyle{icml2023}
\bibliography{bibliography.bib}

\clearpage
\addtocontents{toc}{\protect\setcounter{tocdepth}{2}}

\newpage

\appendix
\onecolumn

\tableofcontents
\newpage

%^%%%%%%%%%%%%%%%%%
%^%%%%%%%%%%%%%%%%%
%^%%%%%%%%%%%%%%%%%
%^%%%%%%%%%%%%%%%%%
%^%%%%%%%%%%%%%%%%%

\section{Proofs of the propositions}
  \subsection{Proof of \Cref{prop:uplink}}\label{sec:proof_uplink}

  From the definition
    \begin{equation*}
      \begin{aligned}
        \bG^{\rm D}_{k+1} 
          &= \frac{1}{n} \sum_{i=1}^{n} \EE\brs{\norm{g^i_{k+1} - \nabla F_i(x_{k+1})}^2}\\
          &= \frac{1}{n} \sum_{i=1}^{n} 
          \EE\brs{\EE\brs{\norm{ g^i_k + \cQ^{\rm D}(\nabla F_i (x_{k+1}) - g^i_k) - \nabla F_i(x_{k+1})}^2\mid x_1,\ldots,x_{k+1}}}\\
          &\leq \frac{1-\alpha_{\rm D}}{n} \sum_{i=1}^{n} 
          \EE\brs{\norm{ g^i_k - \nabla F_i(x_{k+1})}^2}.
           \end{aligned}
    \end{equation*}
    Applying Cauchy-Schwartz and the Lipschitz continuity of the function $\nabla F_i(\cdot)$, we obtain
    \begin{equation*}
      \begin{aligned}
        \bG^{\rm D}_{k+1} 
          &\leq \frac{(1-\alpha_{\rm D})(1+s_{\rm D})}{n} \sum_{i=1}^{n}  \EE\brs{\norm{ g^i_k - \nabla F_i(x_{k})}^2} \\
          & +\frac{(1-\alpha_{\rm D})(1+s_{\rm D}^{-1})}{n} \sum_{i=1}^{n}  \EE\brs{\norm{\nabla F_i(x_{k}) - \nabla F_i(x_{k+1})}^2}\\
          &\leq {(1-\alpha_{\rm D})(1+s_{\rm D})} \bG^{\rm D}_{k}
               + \frac{(1-\alpha_{\rm D})(1+s_{\rm D}^{-1})}{n} \sum_{i=1}^{n}  L_i^2\EE\brs{\norm{x_{k} - x_{k+1}}^2}\\
          &\leq {(1-\alpha_{\rm D})(1+s_{\rm D})} \bG^{\rm D}_{k}
               + {(1-\alpha_{\rm D})(1+s_{\rm D}^{-1})} \bar{L} \EE\brs{\norm{x_{k} - x_{k+1}}^2}\\
          &\leq {(1- p_{\rm D})} \bG^{\rm D}_{k}
            + (1- p_{\rm D})\beta_{\rm D}  \EE\brs{\norm{x_{k} - x_{k+1}}^2}.\\
      \end{aligned}
    \end{equation*}
    This concludes the proof.

    %%%%%%%%%%%%%%%%%%%%%%%%%%%%%

  \subsection{Proof of \Cref{prop:downlink}}
  From the definition
  \begin{equation}
    \begin{aligned}
      \bG^{\rm P}_{k+1} &= L^2\Exp{\norm{w_{k + 1} - x_{k + 1}}^2} \\
      &= L^2\Exp{\norm{w_{k} - x_{k + 1} - \cQ^{\rm P}(w_{k} - x_{k + 1})}^2} \\
      &= (1 - \alpha_{\rm P})L^2\Exp{\norm{w_{k} - x_{k + 1}}^2} \\
      &= (1 - \alpha_{\rm P})L^2\Exp{\norm{w_{k} - x_{k} + x_{k} - x_{k + 1}}^2}\\
      &\leq (1 - \alpha_{\rm P})(1+s)L^2\Exp{\norm{w_{k} - x_{k}}^2} 
      + (1 - \alpha_{\rm P})(1 + s^{-1})L^2\Exp{\norm{x_{k} - x_{k + 1}}^2}.\\
    \end{aligned}
  \end{equation}
  Choosing $s$ small enough, we can make the coefficient $(1 - \alpha_{\rm P})(1+s)$ smaller than one. 
  Thus, defining $p = 1 - (1 - \alpha_{\rm P})(1+s)$, we conclude the proof.

%%%%%%%%%%%%%%%%%%%%%%%%
  
  \subsection{Full statement of \Cref{prop:bidir_incom} and its proof}
  We state now the complete version of \Cref{prop:bidir_incom}.
\begin{proposition}\label{prop:bidir}
  The Lyapunov term $\bG^{\rm D}_k$ of the bidirectional Langevin algorithm satisfies the following recurrent inequality:
  \begin{equation*}
      \bG^{\rm D}_{k+1} \leq \lambda_1 \bG^{\rm D}_k + \lambda_2 \EE\brs{\norm{ x_{k}-  x_{k+1}}^2} + \lambda_3 \bG^{\rm P}_k,
  \end{equation*}
  where $\bG^{\rm P}_k := \bar{L}\EE\brs{\norm{ w_{k}-  x_{k}}^2}$ is the Lyapunov term for P-ELF and 
  \begin{equation}\label{eq:lambda-def}
      \begin{aligned}
        \lambda_1 &= (1-\alpha_{\rm D})(1+s)(1+q); \\
        \lambda_2 &= (1-\alpha_{\rm D})(1+s)(1+q^{-1})(1+u) \bar{L}
        \\ & + \brr{(1-\alpha_{\rm D})(1+s)(1+q^{-1})(1+u^{-1}) + (1+s^{-1})} (1-\alpha_{\rm P}) (1 + w^{-1}) \bar{L};\\
        \lambda_3 &= \brr{(1-\alpha_{\rm D})(1+s)(1+q^{-1})(1+u^{-1}) + (1+s^{-1})} (1-\alpha_{\rm P}) (1 + w). \\
      \end{aligned}
    \end{equation}
    Here, $s,q,u,w$ are any positive numbers.
\end{proposition}
\begin{proof}
From the definition of $\bG^{\rm D}_k$ and  Young's inequality we have 
 \begin{equation*}
      \begin{aligned}
        \bG^{\rm D}_{k+1} 
          &= \frac{1}{n} \sum_{i=1}^{n} \EE\brs{\norm{g^i_{k+1} - \nabla F_i(x_{k+1})}^2}\\
          &= \frac{1}{n} \sum_{i=1}^{n} 
          \EE\brs{\EE\brs{\norm{ g^i_k + \cQ^{\rm D}(\nabla F_i (w_{k+1}) - g^i_k) - \nabla F_i(x_{k+1})}^2\mid x_1,\ldots,x_{k+1}}}\\
           &\leq \frac{1}{n} \sum_{i=1}^{n} \Big\{
          (1+s)\EE\brs{\EE\brs{\norm{ g^i_k + \cQ^{\rm D}(\nabla F_i (w_{k+1}) - g^i_k) - \nabla F_i(w_{k+1})}^2\mid x_1,\ldots,x_{k+1}}} \\
          &+ (1+s^{-1})\EE\brs{\norm{ \nabla F_i (w_{k+1}) - \nabla F_i(x_{k+1}) }^2 } \Big\}.
      \end{aligned}
    \end{equation*}
    The contractivity of $\cQ^{\rm D}$ implies
     \begin{equation*}
      \begin{aligned}
        \bG^{\rm D}_{k+1} 
          &\leq \frac{1}{n} \sum_{i=1}^{n} 
          (1-\alpha_{\rm D})(1+s) \EE\brs{\norm{ g^i_k - \nabla F_i(w_{k+1})}^2} 
          + (1+s^{-1})\bar{L} \EE\brs{\norm{ w_{k+1} - x_{k+1} }^2 }
          \\& \leq \frac{1}{n} \sum_{i=1}^{n} 
          (1-\alpha_{\rm D})(1+s)(1+q) \EE\brs{\norm{ g^i_k - \nabla F_i(x_{k})}^2}
          +(1-\alpha_{\rm D})(1+s)(1+q^{-1}) \EE\brs{\norm{  \nabla F_i(x_{k})- \nabla F_i(w_{k+1})}^2} 
          \\& + (1+s^{-1})\bar{L} \EE\brs{\norm{ w_{k+1} - x_{k+1} }^2 }
          \\& \leq 
          (1-\alpha_{\rm D})(1+s)(1+q) \bG^{\rm D}_k
          +(1-\alpha_{\rm D})(1+s)(1+q^{-1}) \bar{L} \EE\brs{\norm{ x_{k}-  w_{k+1}}^2} 
         + (1+s^{-1}) \bG^{\rm P}_{k+1}.
           \end{aligned}
    \end{equation*}
    Applying Young's inequality to the second term, we deduce
    \begin{equation*}
      \begin{aligned}
        \bar{L}\EE\brs{\norm{ x_{k}-  w_{k+1}}^2}
        &
        \leq (1+u)\bar{L}\EE\brs{\norm{ x_{k}-  x_{k+1}}^2} + (1+u^{-1})\bar{L} \EE\brs{\norm{ x_{k+1}-  w_{k+1}}^2}
        \\&
        = (1+u)\bar{L}\EE\brs{\norm{ x_{k}-  x_{k+1}}^2} 
        + (1+u^{-1}) \bG^{\rm P}_{k+1}.
      \end{aligned}
    \end{equation*}
    Therefore,
    \begin{equation*}
      \begin{aligned}
        \bG^{\rm D}_{k+1} 
          & \leq 
          (1-\alpha_{\rm D})(1+s)(1+q) \bG^{\rm D}_k
        +(1-\alpha_{\rm D})(1+s)(1+q^{-1})(1+u) \bar{L}\EE\brs{\norm{ x_{k}-  x_{k+1}}^2} 
        \\ & + (1-\alpha_{\rm D})(1+s)(1+q^{-1})  (1+u^{-1}) \bG^{\rm P}_{k+1}
          + (1+s^{-1})  {\bG^{\rm P}_{k+1}}.
      \end{aligned}
    \end{equation*}
    Let us now bound the auxiliary term $\bG^{\rm P}_{k+1}$. We notice that $\bG^{\rm P}_k$ is the Lyapunov term of the P-ELF algorithm. Thus, from \Cref{prop:downlink} we have
    \begin{equation}\label{eq:tk-rec}
      \begin{aligned}  
        \bG^{\rm P}_{k+1} &= \bar{L}\EE\brs{\norm{ w_{k+1}-  x_{k+1}}^2} 
        \\ & \leq
          (1-\alpha_{\rm P}) (1+w)\bG^{\rm P}_k
          + (1-\alpha_{\rm P}) (1 + w^{-1})\bar{L}\EE\brs{\norm{  x_{k} -  x_{k+1}}^2}.
      \end{aligned}  
    \end{equation}
    Recalling the definitions of $\lambda_1,\lambda_2,\lambda_3$ we deduce 
    \begin{equation*}
      \bG^{\rm D}_{k+1} \leq \lambda_1 \bG^{\rm D}_k + \lambda_2 \EE\brs{\norm{ x_{k}-  x_{k+1}}^2} + \lambda_3 \bG^{\rm P}_k.
    \end{equation*}  
    This concludes the proof of the proposition.
\end{proof}
    %%%%%

\section{Proofs of the main theorems}

\subsection{Some technical lemmas}

We will use repeatedly, sometimes without even mentioning, a simple inequality  which is a consequence of Young's inequality.
It goes as follows.
\begin{lemma}
  For any two vectors $x,y \in \RR^d$ and any $s > 0$
  \begin{equation*}
    \normsq{x+y} \leq  (1 + s) \normsq{x} + ( 1 + s^{-1}) \normsq{y}.
  \end{equation*}
\end{lemma}
\begin{proof}
  \begin{equation*}
    \begin{aligned}
      \normsq{x+y} &=  \normsq{x} + 2\inner{x}{y} +  \normsq{y}\\
      &\leq  (1 + s) \normsq{x} + ( 1 + s^{-1}) \normsq{y}.
    \end{aligned}
  \end{equation*}
  The second passage is due to Young's inequality.
\end{proof}

We also use two lemmas from the literature, which we present below without proofs. 
The first one is an instance of Grönwall's inequality in its integral form. 
Its proof can be found in \cite{amann2011ordinary}.
\begin{lemma}[Grönwall's Inequality]\label{lem:gronwall}
  Assume $\phi, B:[0, T] \rightarrow \mathbb{R}$ are bounded non-negative measurable function and $C:[0, T] \rightarrow \mathbb{R}$ is a non-negative integrable function with the property that
  \begin{equation}
    \label{eq:grrrr1}
    \phi(t) \leq B(t)+\int_{0}^{t} C(\tau) \phi(\tau) \rmd \tau \quad \text { for all } t \in[0, T].
  \end{equation}
  Then
  \begin{equation*}
    \label{eq:grrrr2}
    \phi(t) \leq B(t)+\int_{0}^{t} B(s) C(s) \exp \left(\int_{s}^{t} C(\tau) \rmd \tau\right) \rmd s \quad \text { for all } t \in[0, T].
  \end{equation*}
\end{lemma}

The second is a technical lemma borrowed from \citet{chewi2021analysis}.
\begin{lemma}
  \label{lem:Chew}
  Suppose that $\nabla F$ is $L$-Lipschitz. Then for any probability measure $\nu$, the following inequality is satisfied:  
  \begin{equation*}
    \label{eq:chew}
    \mathbb{E}_{\nu}\left[\|\nabla F\|^{2}\right] \leq \mathbb{E}_{\nu}\left[\left\|\nabla \log \left( \frac{\nu}{ \pi}\right)\right\|^{2}\right]+2 d L =\FS{\nu}+2dL.
  \end{equation*}
\end{lemma}

\subsection{Proof of Theorem \ref{thm:ef21-lmc}}\label{sec:proof_ef21-lmc}

We follow the scheme described in \Cref{sec:proof-scheme}. 
Let us recall the initial setting first.
The update rule of both D-ELF and P-ELF can be abstractly defined by
\begin{equation*}
  x_{k+1} = x_k - \gamma g_k + \sqrt{2\gamma} Z_k.
\end{equation*}
The vector $g_k$ is a stochastic estimator of the potential function's gradient at the $k$-th iterate: 
$\nabla F(x_k)$. 
On the other hand, for each $k$ the next iteration can be computed using the following 
SDE:
\begin{equation}
  \rmd y_t =  - g_k\rmd t + \sqrt{2}\rmd B_t, 
\end{equation}  
with  $y_0 = x_k$ and  $t \in [0,\gamma]$.
Then, as shown in \Cref{sec:proof-scheme}, $y_{\gamma} = x_{k+1}$.  
Denote by $\rho_t$ the distribution of $y_t$. \Cref{lem:fp-transformation} yields:
  \begin{equation}\label{eq:lem1-proof}
  \begin{aligned}
    \dif{\KL{\rho_t}} 
    &\leq -\frac{3}{4}\FS{\rho_t}+\Exp{\normsq{f_{\xi_k}(y_0)-\nabla F(y_t)}}\\
    &\leq -\frac{3}{4}\FS{\rho_t}+\Exp{\normsq{g_k-\nabla F(y_t)}}.
  \end{aligned}
  \end{equation}

  \paragraph{The proof for D-ELF:} 
  The Lyapunov term for  the D-ELF algorithm is defined as
\begin{equation*}
    \bG^{\rm D}_k := \frac{1}{n} \sum_{i}^{n} \EE\brs{\norm{g^i_k - \nabla F_i(x_k)}^2}.  
\end{equation*}

Next lemma bounds the second term in \eqref{eq:lem1-proof} using $\bG^{\rm D}_k$.

\begin{lemma}\label{lem:dif-kl-fi}
  If $f_{\xi_k}(x_k)$ is the gradient estimator $g_k$ from \Cref{alg:ef21_langevin}, then 
  $\rho_t$ satisfies
  \begin{equation}
    \label{eq:ef21-lmc-finahave}
    \dif{\KL{\rho_t}}\leq -\frac{3}{4}\FS{\rho_t}+2L^2\Exp{\normsq{x_{k+1}-x_k}}+2\bG^{\rm D}_k.
  \end{equation} 
\end{lemma}

% The next lemma concludes the proof of the theorem for a certain class of stochastic gradient LMC algorithms. 
% In particular, this class covers the EF21 Langevin algorithm.
% \begin{lemma}\label{lem:recursion}
%   Let $\rho_t$ be defined as in \Cref{sec:proof-scheme}. Suppose also  that \eqref{eq:kl-gk} is satisfied with  
%   $\bG^{\rm D}_k$, for which the following recurrent inequality is true:
%   \begin{equation}\label{eq:gk-rec}
%       \bG^{\rm D}_{k+1} \leq (1 - p)\bG^{\rm D}_k + (1 - p)\beta_{\rm D} \EE\brs{\norm{x_{k + 1} - x_k}^2},
%   \end{equation}
%   with $\beta_{\rm D} > 0$. Then 
%   \begin{equation*}
%     \begin{aligned}
%       \bH_{k+1} &\leq e^{-\mu \gamma}\bH_k + \frac{1-e^{-\mu \gamma}}{\mu}\tau,
%     \end{aligned}
% \end{equation*}
% where 
% \begin{equation*}
%   \begin{aligned}
%     \bH_k &:= \KL{\rho_{k\gamma}}+\frac{1-e^{-\mu \gamma}}{\mu}C\bG^{\rm D}_{k}\\
%     \tau  &:= {\left(2L^2+C(1-p) \beta_{\rm D}\right)\left(16L\gamma^2d+4d\gamma\right)}
%   \end{aligned}
% \end{equation*}
%  and $\gamma < ?$.
% \end{lemma}

Let us now add $C\bG^{\rm D}_{k+1}$ to both sides of the inequality \eqref{eq:ef21-lmc-finahave}, where $C > 0$ is a constant to be determined later:
\begin{equation*}
  \begin{aligned}
    \dif{\KL{\rho_t}} + C   \bG^{\rm D}_{k+1}
    &\leq -\frac{3}{4}\FS{\rho_t}+2L^2\Exp{\normsq{x_{k+1}-x_k}}+2 \bG^{\rm D}_k+C  \bG^{\rm D}_{k+1}.
  \end{aligned}
\end{equation*}
Combining \Cref{prop:uplink} and \eqref{eq:ef21-lmc-kk+1} we deduce
\begin{equation*}
  \begin{aligned}
  \dif{\KL{\rho_t}} + C \bG^{\rm D}_{k+1}
    &\leq-\frac{3}{4}\FS{\rho_t} + 2L^2\Exp{\normsq{x_{k+1}-x_k}}+2 \bG^{\rm D}_k\\
    &+C \left((1-p)\bG^{\rm D}_k+(1-p){ }\beta_{\rm D}\Exp{\normsq{x_{k+1}-x_k}}\right)\\
    &=-\frac{3}{4}\FS{\rho_t}+\left(2L^2+C(1-p) \beta_{\rm D}\right)\Exp{\normsq{x_{k+1}-x_k}}+\left(2+C(1-p)\right) \bG^{\rm D}_k\\
   \end{aligned}
\end{equation*}
The lemma below bounds the term $\Exp{\normsq{x_{k+1}-x_k}}$.
\begin{lemma}\label{lem:exp_x_k+1_x_k}
  If $\gamma\leq\frac{1}{2\sqrt{2}L}$, then the iterates of the stochastic LMC algorithm \eqref{eq:stoch-lmc} satisfy the following inequality, where $\bG^{\rm D}_k$ is the Lyapunov term of D-ELF algorithm defined in \eqref{eq:gk-ef21}: 
  \begin{equation}
    \label{eq:ef21-lmc-kk+1}
    \Exp{\normsq{x_{k+1}-x_k}}\leq 8\gamma^2\Exp{\normsq{\nabla F(y_t)}}+4\gamma^2 \bG^{\rm D}_k+4d\gamma.
  \end{equation}
  \end{lemma}

\Cref{lem:exp_x_k+1_x_k} yields the following
\begin{equation*}
  \begin{aligned}
  \dif{\KL{\rho_t}} + C \bG^{\rm D}_{k+1}
    &{\leq}-\frac{3}{4}\FS{\rho_t}
    +\left(2L^2+C(1-p) \beta_{\rm D}\right)\left(8\gamma^2\Exp{\normsq{\nabla F(y_t)}}+4\gamma^2 \bG^{\rm D}_k 
    + 4d\gamma\right)\\
    &\quad+\left(2+C(1-p)\right) \bG^{\rm D}_k.
  \end{aligned}
\end{equation*}
% \begin{equation*}
%   \begin{aligned}
%   \dif{\KL{\rho_t}} + C \bG^{\rm D}_{k+1}
%     &\leq-\frac{3}{4}\FS{\rho_t} + 2L^2\Exp{\normsq{x_{k+1}-x_k}}+2 \bG^{\rm D}_k\\
%     &+C \left((1-p)\bG^{\rm D}_k+(1-p){ }\beta_{\rm D}\Exp{\normsq{x_{k+1}-x_k}}\right)\\
%     &=-\frac{3}{4}\FS{\rho_t}+\left(2L^2+C(1-p) \beta_{\rm D}\right)\Exp{\normsq{x_{k+1}-x_k}}+\left(2+C(1-p)\right) \bG^{\rm D}_k\\
%     &\overset{\eqref{eq:ef21-lmc-kk+1}}{\leq}-\frac{3}{4}\FS{\rho_t}
%     +\left(2L^2+C(1-p) \beta_{\rm D}\right)\left(8\gamma^2\Exp{\normsq{\nabla F(y_t)}}+4\gamma^2 \bG^{\rm D}_k 
%     + 4d\gamma\right)\\
%     &\quad+\left(2+C(1-p)\right) \bG^{\rm D}_k.
%   \end{aligned}
% \end{equation*}
Let us now apply \Cref{lem:Chew} to the right-hand side. We obtain
\begin{equation*}
  \begin{aligned}
    \dif{\KL{\rho_t}} + C  \bG^{\rm D}_{k+1}
    &\leq-\frac{3}{4}\FS{\rho_t}+\left(2L^2+C(1-p) \beta_{\rm D}\right)
    \left(8\gamma^2\left(\FS{\rho_t}+2dL\right)+4\gamma^2\bG^{\rm D}_k+4d\gamma\right)\\
    &\quad+\left(2+C(1-p)\right)\bG^{\rm D}_k\\
    &=-\left(\frac{3}{4}-8\gamma^2\left(2L^2+C(1-p) \beta_{\rm D}\right)\right)\FS{\rho_t}\\
    & \quad +\left(8L^2\gamma^2+C(1-p)\left(4 \gamma^2\beta_{\rm D}+1\right)+2\right)\bG^{\rm D}_k\\
    &\quad+{\left(2L^2+C(1-p) \beta_{\rm D}\right)\left(16L\gamma^2d+4d\gamma\right)}.\\
   \end{aligned}
  \end{equation*}
  From the definition of $\tau$ we obtain the following:
  \begin{equation}
    \begin{aligned}
      \label{eq:ef21-lmc-jp}
      \dif{\KL{\rho_t}} + C  \bG^{\rm D}_{k+1}  
      &\leq  -\left(\frac{3}{4}-8\gamma^2\left(2L^2+C(1-p) \beta_{\rm D}\right)\right)\FS{\rho_t}\\
      &\quad+\left(8L^2\gamma^2+C(1-p)\left(4 \gamma^2\beta_{\rm D}+1\right)+2\right)\bG^{\rm D}_k+\tau.\\
    \end{aligned}
  \end{equation}
Let $C=\left(8L^2\gamma^2+C(1-p)\left(4 \gamma^2\beta_{\rm D}+1\right)+2\right)e^{\mu \gamma}$.  
Solving this linear equation w.r.t. $C$, we get
\begin{equation}
  \label{eq:ef21-lmc-newC}
  C=\frac{8L^2\gamma^2 + 2}{e^{-\mu \gamma}-(1-p)\left(4 \gamma^2\beta_{\rm D}+1\right)}.
\end{equation}
%we need 
%\begin{equation}\label{eq:ef21-lmc-sfskfj}
% C>0,\quad \frac{3}{4}-8\gamma^2\left(2L^2+C(1-p) \beta_{\rm D}\right)\geq 
% \frac{1}{2}.
%\end{equation}

Without loss of generality we may assume that  $\mu \gamma <1 $  and thus we have $e^{\mu \gamma}\leq 1+2\mu \gamma$. 
In order for $C$ to be positive, we need to assure that
\begin{equation*}
  1-(1-p)\left(4 \beta_{\rm D} \gamma^2+1\right)\left(1+2\mu \gamma\right)>0.
\end{equation*}
The latter is equivalent to 
\begin{equation*}
  {\frac{1-p}{p}8 \mu\beta_{\rm D} \gamma^3} + {\frac{1-p}{p}4 \beta_{\rm D} \gamma^2} + {\frac{1-p}{p}2\mu \gamma} <1.
\end{equation*}
A simple solution to this inequality is to make all three terms smaller than $1/3$.
The latter is equivalent to
\begin{equation}\label{old_re:ree}
  \gamma< \min \left\{\left(\frac{p}{24 \mu\beta_{\rm D}(1-p)}\right)^{1/3}, \left(\frac{p}{12 \beta_{\rm D} (1-p)}\right)^{1/2},\frac{p}{6\mu(1-p)} \right\}.
\end{equation}
On the other hand, we will require the coefficient of $\FS{\rho_t}$ in \eqref{eq:ef21-lmc-jp} to be negative. This is to ensure contraction. That means
\begin{equation*}
  8\gamma^2\left(2L^2 + C(1-p)  \beta_{\rm D}\right)=8\gamma^2\left(2L^2+\frac{(8L^2\gamma^2 + 2) (1-p)  \beta_{\rm D}}{e^{-\mu \gamma}-(1-p)\left(4 \gamma^2\beta_{\rm D}+1\right)}\right)\leq\frac{1}{4}.
\end{equation*}
Solving this inequality we get 
\begin{equation}\label{eq:ef21-lmc-hbounddddd}
  \gamma\leq\frac{1}{2}\sqrt{\frac{1-(1-p)e^{\mu \gamma}}{ \brr{16+(1-p)(17\beta_{\rm D}-16)e^{\mu \gamma}}}}.
\end{equation}
% If $\beta=1$, we  need 
% \begin{equation}
%   \label{eq:ef21-lmc-hbounuuu}
%   {\gamma\leq\frac{1}{10L}\sqrt{\frac{p}{1+\beta_{\rm D}}}}\leq\frac{1}{2L}\sqrt{\frac{p}{17\beta_{\rm D}(1-p)}}\leq\frac{1}{2L}\min\{\sqrt{\frac{p}{16p+17\beta_{\rm D}(1-p)}},\sqrt{\frac{p}{\beta_{\rm D}(1-p)}}\}
% \end{equation}
% to guarantee $C\geq 0$ and \eqref{eq:ef21-lmc-hbounddddd}.
From \eqref{old_re:ree}, we know that $\gamma< \frac{p}{6\mu(1-p)}$, so $e^{\mu \gamma}\leq 1+2\mu \gamma\leq
1+\frac{p}{3(1-p)}$. Inserting this upper bound  into \eqref{eq:ef21-lmc-hbounddddd}, we get a lower bound on the right hand side. That is 
\begin{equation*}
  \begin{aligned}
    \frac{1}{2}\sqrt{\frac{2p}{ \brs{17\beta_{\rm D}(3-2p)+32p}}}
    &=\frac{1}{2}\sqrt{\frac{1-(1-p)(1+\frac{p}{3(1-p)})}{  \brr{16+(1-p)(17\beta_{\rm D}-16)(1+\frac{p}{3(1-p)})}}}\\
    &\leq\frac{1}{2}\sqrt{\frac{1-(1-p)e^{\mu \gamma}}{ \brr{16+(1-p)(17\beta_{\rm D}-16)e^{\mu \gamma}}}}.
  \end{aligned}
\end{equation*}
So we need 
\begin{equation*} 
    \gamma<\min \left\{\frac{1}{2}\sqrt{\frac{2p}{ \brs{17\beta_{\rm D}(3-2p)+32p}}},\left(\frac{p}{24 \mu\beta_{\rm D}(1-p)}\right)^{1/3}, \left(\frac{p}{12 \beta_{\rm D} (1-p)}\right)^{1/2},\frac{p}{6\mu(1-p)}  \right\}.
\end{equation*}

We can further simplify this inequality. 
The first and third terms  are larger than $a:=\frac{1}{14}\sqrt{\frac{p}{ (1+\beta_{\rm D})}},$ while as the fourth term is larger than $b := \frac{p}{6\mu}$.  
On the other hand, $\min\{a,b\}$ is less than the second term. 
Indeed,
\begin{equation*}
  \min\{a,b\}\leq a^{2/3}b^{1/3}=\left(\frac{p^2}{1176 \mu(1+\beta_{\rm D})}\right)^{1/3}
  \leq \left(\frac{p}{24 \mu\beta_{\rm D}(1-p)}\right)^{1/3}.
\end{equation*}

Summing up, we obtain the following bound on the step-size that guarantees $C\geq0$ and \eqref{eq:ef21-lmc-hbounddddd}:
\begin{equation*}
    % \label{eq:ef21-lmc-rererererer}
    \gamma\leq\min \left\{\frac{1}{14}\sqrt{\frac{p}{ (1+\beta_{\rm D})}},\frac{p}{6\mu}\right\}.
\end{equation*}
Therefore, the above the conditions are satisfies. This yields the following:
\begin{equation}
  \label{eq:ef21-lmc-finnnnn}
  \dif{\KL{\rho_t}}+C\bG^{\rm D}_{k+1}\leq -\frac{1}{2}\FS{\rho_t} + e^{-\mu \gamma} C\bG^{\rm D}_k + C\tau.
\end{equation}

Since $\pi$ satisfies Log-Sobolev inequality, we deduce
\begin{equation}
  \label{eq:ef21-lmc-LSILA}
  \dif{\KL{\rho_t}}+C\bG^{\rm D}_{k+1}\leq-\mu\KL{\rho_t}+e^{-\mu \gamma}C\bG^{\rm D}_k+\tau.
\end{equation}
One may check that the  equivalent integral form of \eqref{eq:ef21-lmc-LSILA} 
 satisfies \eqref{eq:grrrr1} with $\phi(t)=\KL{\rho_t},~B(t)=\left(e^{- \mu \gamma}C\bG^{\rm D}_k-C\bG^{\rm D}_{k+1}+\tau\right)t+\KL{\rho_{k\gamma}},~C(t)=-\mu$. Therefore,  from \Cref{lem:gronwall} we deduce 
\begin{equation*}
  \label{old_eq;KLKLKL}
  \KL{\rho_t}\leq e^{-\mu t}\KL{\rho_{k\gamma}}+\frac{1-e^{-\mu t}}{\mu}\left(e^{-\mu \gamma}C\bG^{\rm D}_k-C\bG^{\rm D}_{k+1}+\tau\right),
\end{equation*}
let $t=\gamma$ and $\beta=e^{\mu \gamma}$, then we have
\begin{equation}
  \begin{aligned}
    \KL{\rho_{(k+1)\gamma}}+\frac{1-e^{-\mu \gamma}}{\mu}C\bG^{\rm D}_{k+1}&\leq e^{-\mu \gamma}\left(\KL{\rho_{k\gamma}}+e^{\mu \gamma}\frac{1-e^{-\mu \gamma}}{\mu}\beta^{-1}C\bG^{\rm D}_k\right)+\frac{1-e^{-\mu \gamma}}{\mu}\tau\\
    &=e^{-\mu \gamma}\left(\KL{\rho_{k\gamma}}+\frac{1-e^{-\mu \gamma}}{\mu}C\bG^{\rm D}_{k}\right)+\frac{1-e^{-\mu \gamma}}{\mu}\tau.
  \end{aligned}
\end{equation} 

Repeating this step for $k=0,1,2,\cdots,K-1$, we obtain 
\begin{equation*}
    \begin{aligned}
      \bH_K &\leq e^{- K \mu \gamma}\bH_0 + \frac{1-e^{-K\mu \gamma}}{\mu}\tau.
    \end{aligned}
\end{equation*}
This proves \Cref{thm:ef21-lmc} for D-ELF.

\paragraph{The proof for P-ELF:} 

The gradient estimator $\nabla f_{\xi_k}(x_k)$ in this case is equal to 
\begin{equation*}
  \nabla f_{\xi_k}(x_k) = \nabla F(w_k) =\frac{1}{n} \sum_{i=1}^n \nabla F_i(w_k). 
\end{equation*}
From $L_i$-smoothness of the $i$-th component function $F_i$ we deduce the following relation:
\begin{equation} \label{eq:gd-gp}
  \begin{aligned}
  \bG^{\rm D}_k &=   \frac{1}{n} \sum_{i}^{n} \EE\brs{\norm{\nabla F_i(w_k) - \nabla F_i(x_k)}^2}\\
   & \leq  \frac{1}{n} \sum_{i}^{n} \EE\brs{L_i^2 \norm{ w_k - x_k}^2}\\
   & = \bG^{\rm P}_k.
  \end{aligned}
\end{equation}
Therefore, combining this inequality with \Cref{lem:dif-kl-fi} we obtain
\begin{equation*}
  \begin{aligned}
    \dif{\KL{\rho_t}} 
    &\leq -\frac{3}{4}\FS{\rho_t}+2L^2\Exp{\normsq{x_{k+1}-x_k}}+2\bG^{\rm D}_k \\
    &\leq -\frac{3}{4}\FS{\rho_t}+2L^2\Exp{\normsq{x_{k+1}-x_k}}+2  \bG^{\rm P}_k.
  \end{aligned}
\end{equation*}
The latter means that we can repeat exactly the rest of the proof of D-ELF by replacing $\bG^{\rm D}_k$ with $ \bG^{\rm P}_k$ and using \Cref{prop:downlink} instead of \Cref{prop:uplink}.
Therefore, 
\begin{equation*}
    \begin{aligned}
      \bH_K &\leq e^{- K \mu \gamma}\bH_0 + \frac{1-e^{-K\mu \gamma}}{\mu}\tau.
    \end{aligned}
\end{equation*}
This concludes the proof of \Cref{thm:ef21-lmc}.

\subsection{Proof of \Cref{thm:bidir}}\label{sec:proof_bidir}
We recall the definition of the Lyapunov term $\bG^{\rm D}_k$:
\begin{equation*}
  \bG^{\rm D}_k := \frac{1}{n} \sum_{i}^{n} \EE\brs{\norm{g^i_k - \nabla F_i(x_k)}^2}.
\end{equation*}

As described in \Cref{sec:proof-scheme}, we use the interpolation proof scheme.
That is for the $k$-th iteration we define the process $y_t$ as in \eqref{eq:ld-stoch}.
Thus, from \Cref{lem:fp-transformation} we have
\begin{equation*}
  \begin{aligned}
    \dif{\KL{\rho_t}} 
     & \leq -\frac{3}{4}\FS{\rho_t}+\Exp{\normsq{f_{\xi_k}(y_0)-\nabla F(y_t)}}
    \\ & = -\frac{3}{4}\FS{\rho_t}+\Exp{\normsq{ g_0 - \nabla F(y_t)}}.
  \end{aligned}
\end{equation*}

Combining this with \Cref{prop:bidir} and \eqref{eq:tk-rec}, we obtain
\begin{equation*}
  \begin{aligned}
    \dif{\KL{\rho_t}} &+ C \bG^{\rm D}_{k+1} + D \bG^{\rm P}_{k+1}
    \\ & \leq 
    -\frac{3}{4}\FS{\rho_t}+2L^2\Exp{\normsq{x_{k+1}-x_k}}+2 \bG^{\rm D}_k+C  \bG^{\rm D}_{k+1}  + D \bG^{\rm P}_{k+1}
    \\ & \leq 
    -\frac{3}{4}\FS{\rho_t}+2L^2\Exp{\normsq{x_{k+1}-x_k}}+2 \bG^{\rm D}_k+C \brr{\lambda_1 \bG^{\rm D}_k + \lambda_2 \EE\brs{\norm{ x_{k}-  x_{k+1}}^2} + \lambda_3 \bG^{\rm P}_k } 
    \\ & + D \brr{(1-\alpha_{\rm P}) (1+w)\bG^{\rm P}_k
              + (1-\alpha_{\rm P}) (1 + w^{-1})\bar{L}\EE\brs{\norm{  x_{k} -  x_{k+1}}^2} }
    \\ & = 
    -\frac{3}{4}\FS{\rho_t} + \brr{2L^2 + C\lambda_2 + D(1-\alpha_{\rm P}) (1 + w^{-1}) \bar{L}}\EE\brs{\norm{  x_{k} -  x_{k+1}}^2}
     \\ & + (2 + C\lambda_1)\bG^{\rm D}_k  + \brr{C\lambda_3 +  D(1-\alpha_{\rm P}) (1+w)}\bG^{\rm P}_k.
      \end{aligned}
\end{equation*}
\Cref{lem:exp_x_k+1_x_k} yields
\begin{equation*}
  \Exp{\normsq{x_{k+1}-x_k}}\leq 8\gamma^2\Exp{\normsq{\nabla F(y_t)}}+4\gamma^2 \bG^{\rm D}_k + 4d\gamma,
\end{equation*}
for $\gamma < \nicefrac{1}{8L}$. 
The latter condition on the step-size is a consequence of our assumptions from the statement of \Cref{thm:bidir}. 
Therefore,
\begin{equation*}
  \begin{aligned}
    \dif{\KL{\rho_t}} &+ C \bG^{\rm D}_{k+1} + D \bG^{\rm P}_{k+1}
     \\ & \leq 
    -\frac{3}{4}\FS{\rho_t} + \brr{2L^2 + C\lambda_2 + D(1-\alpha_{\rm P}) (1 + w^{-1}) \bar{L}}\brr{8\gamma^2\Exp{\normsq{\nabla F(y_t)}}+4\gamma^2 \bG^{\rm D}_k + 4d\gamma}
     \\ & + (2 + C\lambda_1)\bG^{\rm D}_k  + \brr{C\lambda_3 +  D(1-\alpha_{\rm P}) (1+w)}\bG^{\rm P}_k.
  \end{aligned}
\end{equation*}

Applying \Cref{lem:Chew} we deduce 
\begin{equation*}
  \begin{aligned}
    \dif{\KL{\rho_t}} &+ C \bG^{\rm D}_{k+1} + D \bG^{\rm P}_{k+1}
     \\ & \leq 
    -\frac{3}{4}\FS{\rho_t} + \brr{2L^2+ C\lambda_2  + D(1-\alpha_{\rm P}) (1 + w^{-1}) \bar{L}}\brr{8\gamma^2\brs{\FS{\rho_t}+2dL}+4\gamma^2 \bG^{\rm D}_k + 4d\gamma}
     \\ & + (2 + C\lambda_1)\bG^{\rm D}_k  + \brr{C\lambda_3 +  D(1-\alpha_{\rm P}) (1+w)}\bG^{\rm P}_k
    \\ & =
    \brr{-\frac{3}{4} + 8\gamma^2 \brr{2L^2+ C\lambda_2  + D(1-\alpha_{\rm P}) (1 + w^{-1}) \bar{L}} }\FS{\rho_t} 
     \\ & + \brc{2 + C\lambda_1 + 4\gamma^2\brr{2L^2 + C\lambda_2  + D(1-\alpha_{\rm P}) (1 + w^{-1}) \bar{L}}}\bG^{\rm D}_k + \brr{C\lambda_3 + D(1-\alpha_{\rm P}) (1+w)}\bG^{\rm P}_k
      \\ & + {\brr{2L^2 + C\lambda_2  + D(1-\alpha_{\rm P}) (1 + w^{-1}) \bar{L}}\brr{16\gamma^2{dL}+ 4d\gamma}}.
  \end{aligned}
\end{equation*}
Let us choose $C$ and $D$ to satisfy 
\begin{equation}\label{eq:CD-def}
  \begin{aligned}
    C &= \frac{2.125}{e^{-\mu \gamma} -\lambda_1} \quad \text{and} \quad
    D =\frac{ {2.125}\lambda_3}{ \brr{e^{-\mu \gamma} -\lambda_1}\brr{e^{-\mu \gamma}  - (1-\alpha_{\rm P}) (1+w) }},
  \end{aligned}
\end{equation}
where $\mu$ is the constant from Log-Sobolev inequality.
In order for $C$ and $D$ to be positive we need $\lambda_1$ and  $(1-\alpha_{\rm P}) (1+w)$ to be smaller than  $e^{-\mu\gamma}$. 
We will choose $w$ and  $q = s$ as solutions to the following equations:
\begin{equation}\label{eq:qw}
   \begin{aligned}
      \lambda_1 = (1 - \alpha_{\rm D})(1 + q)^2 &= 1 - \frac{\alpha_{\rm D}}{2}; \\ 
      (1-\alpha_{\rm P}) (1+w) &= 1 - \frac{\alpha_{\rm P}}{2}.
   \end{aligned}
 \end{equation} 
 Then, 
 \begin{equation}\label{eq:gamma-alpha}
   e^{-\mu\gamma} > 1 - \mu\gamma > \max\brc{1 - \nicefrac{\alpha_{\rm D}}{4},1 -\nicefrac{ \alpha_{\rm P}}{4}}
 \end{equation} thus the denominators are positive. 
 Furthermore, 
 \begin{equation*}
   D =\frac{ {2.125}\lambda_3}{ \brr{e^{-\mu \gamma} -\lambda_1}\brr{e^{-\mu \gamma}  - (1-\alpha_{\rm P}) (1+w) }}
   \leq \frac{4C\lambda_3}{\alpha_{\rm P}}.
 \end{equation*}
 Recall that the definitions of $\lambda_2 $ and $\lambda_3$ are given in \eqref{eq:lambda-def}.
    Since $(1-\alpha_{\rm P}) (1 + w) < 1$, from the definition of $\lambda_3$ we have 
     \begin{equation*}
      \begin{aligned}
        \lambda_3
        &= \brr{2(1-\alpha_{\rm D})(1+q)(1+q^{-1}) + (1+q^{-1})} (1-\alpha_{\rm P}) (1 + w)\\
        &\leq { \brr{2(1-\alpha_{\rm D})(2+q+q^{-1}) + (1+q^{-1})} (1-\alpha_{\rm P}) (1 + w)} \\
        &\leq  \brr{2(1-\alpha_{\rm D})(2+q+q^{-1}) + (1+q^{-1})}.
      \end{aligned}
    \end{equation*}
    Therefore, \eqref{eq:lambda-def} implies
     \begin{equation*}
      \begin{aligned}
        \lambda_3(1-\alpha_{\rm P}) (1 + w^{-1}) \bar{L} 
        &= \brr{2(1-\alpha_{\rm D})(2+q+q^{-1}) + (1+q^{-1})} (1-\alpha_{\rm P}) (1 + w^{-1}) \bar{L} \leq
        \lambda_2.
      \end{aligned}
    \end{equation*}
    Thus, 
    \begin{equation*}
    \begin{aligned}
      \gamma^2 \brr{2L^2 + C\lambda_2 + D(1-\alpha_{\rm P}) (1 + w^{-1}) \bar{L}}
      &\leq   \gamma^2 \brr{2L^2 + C\lambda_2 + \frac{4C\lambda_3}{\alpha_{\rm P}} (1-\alpha_{\rm P}) (1 + w^{-1}) \bar{L}}
       \\
      & \leq \gamma^2 \brr{2L^2 + C\lambda_2 +   \frac{4C\lambda_2}{\alpha_{\rm P}}}\\
      & \leq \gamma^2 \brr{2L^2 +   \frac{5C\lambda_2}{\alpha_{\rm P}}}.
    \end{aligned}
  \end{equation*}
  The next lemma bounds the right hand side of the previous inequality by a constant. 
  This will allow us to get a negative coefficient for the $\FS{\rho_t}$ term. 
\begin{lemma}\label{lem:gamma-cond}
  Suppose $u=1$, $q = s$, $C$ and $D$ are defined as in \eqref{eq:CD-def}. Let  \eqref{eq:qw} and \eqref{eq:gamma-alpha}
  also be true.
  Under the assumptions of \Cref{thm:bidir}, the step-size satisfies the following inequality:  
  \begin{equation*}
    \gamma^2 \brr{2L^2 +   \frac{5C\lambda_2}{\alpha_{\rm P}}} < \frac{1}{32}.    
  \end{equation*}
\end{lemma}
The proof is postponed to \Cref{proof:lem:gamma-cond}.
Applying \Cref{lem:gamma-cond} to the first term we finally obtain the following recurrent inequality
  \begin{equation*}
  \begin{aligned}
    \dif{\KL{\rho_t}} &+ C \bG^{\rm D}_{k+1} + D \bG^{\rm P}_{k+1}\\
    & \leq -\frac{1}{2}\FS{\rho_t} 
    +  \brr{2.125 + C\lambda_1}\bG^{\rm D}_k + \brr{C\lambda_3 + D(1-\alpha_{\rm P}) (1+w)}\bG^{\rm P}_k
      \\ & + {\brr{2L^2 + C\lambda_2  + D(1-\alpha_{\rm P}) (1 + w^{-1}) \bar{L}}\brr{16\gamma^2{dL}+ 4d\gamma}}\\
     & \leq -\frac{1}{2}\FS{\rho_t} 
    +  \brr{2.125 + C\lambda_1}\bG^{\rm D}_k + \brr{C\lambda_3 + D(1-\alpha_{\rm P}) (1+w)}\bG^{\rm P}_k
      \\ & + \underbrace{\brr{2L^2 +   \frac{5C\lambda_2}{\alpha_{\rm P}}}\brr{16\gamma^2{dL}+ 4d\gamma}}_{:=\tau}.
    \end{aligned}
\end{equation*}

Then, inserting the values of $C$ and $D$, we get
\begin{equation*}
  \begin{aligned}
    \dif{\KL{\rho_t}} &+ C \bG^{\rm D}_{k+1} + D \bG^{\rm P}_{k+1}
       \leq 
    -\frac{1}{2} \FS{\rho_t} 
    + e^{-\mu \gamma} C \bG^{\rm D}_k + e^{-\mu \gamma} D\bG^{\rm P}_k + \tau.
  \end{aligned}
\end{equation*}
Let us now apply LSI: 
\begin{equation*}
  \begin{aligned}
    \dif{\KL{\rho_t}} &+ C \bG^{\rm D}_{k+1} + D \bG^{\rm P}_{k+1}& \leq - {\mu}\KL{\rho_t} + e^{-\mu \gamma} C \bG^{\rm D}_k + e^{-\mu \gamma} D\bG^{\rm P}_k + \tau.
  \end{aligned}
\end{equation*}
Hence, the derivative of the function $\KL{\rho_t}$ is bounded by itself plus a term that does not depend on $t$. \Cref{lem:gronwall} yields the following:
\begin{equation*}
  \begin{aligned}
    {\KL{\rho_t}} & \leq   e^{-\mu t}  \KL{\rho_0} 
    +\frac{1 - e^{-\mu t} }{\mu} \brr{e^{-\mu \gamma} C \bG^{\rm D}_k + e^{-\mu \gamma} D\bG^{\rm P}_k -  C \bG^{\rm D}_{k+1} - D \bG^{\rm P}_{k+1} + \tau}.
  \end{aligned}
\end{equation*}
In particular, for $t = \gamma$, we have
\begin{equation*}
  \begin{aligned}
    {\KL{\rho_{\gamma}}}  +\frac{1 - e^{-\mu \gamma} }{\mu}\brr{C \bG^{\rm D}_{k+1} + D \bG^{\rm P}_{k+1}} 
    & \leq   e^{-\mu \gamma}  \KL{\rho_0} 
    +\frac{1 - e^{-\mu \gamma} }{\mu} \brr{e^{-\mu \gamma} C \bG^{\rm D}_k + e^{-\mu \gamma} D\bG^{\rm P}_k  + \tau}
    \\ & =   e^{-\mu \gamma} \brs{ \KL{\rho_0} 
        +\frac{1 - e^{-\mu \gamma} }{\mu} \brr{  C \bG^{\rm D}_k +   D\bG^{\rm P}_k} } + \frac{1 - e^{-\mu \gamma} }{\mu}\tau.
  \end{aligned}
\end{equation*}
We first recall that $\rho_{\gamma} = \nu_{K+1}$ and $\rho_0 = \nu_K$. 
Repeating this inequality recurrently we deduce the following bound:
\begin{equation*}
  \begin{aligned}
    {\KL{\nu_K}}  +\frac{1 - e^{-\mu \gamma} }{\mu}\brr{C \bG^{\rm D}_{K} + D \bG^{\rm P}_{K}} 
    & \leq e^{-\mu \gamma K} \brs{ \KL{\rho_0} 
    +\frac{1 - e^{-\mu \gamma} }{\mu} \brr{  C \bG^{\rm D}_0 +   D\bG^{\rm P}_0} } + \frac{\tau }{\mu}.
  \end{aligned}
\end{equation*}
This concludes the proof of \Cref{thm:bidir}. 

\begin{remark}
  One may check, that repeating the analysis for the case when one of the compressor operators $(\alpha = 1)$ is the identity, we will recover the previously known algorithms.
\end{remark}

\subsection{Proof of \Cref{corr:belf-conv}}\label{sec:proof-corr-belf}
  First let us upper bound $\tau$. Similar to the proof of \Cref{corr:delf-conv}, $\brr{16\gamma^2{dL}+ 4d\gamma} < 5d\gamma$. 
  Thus,
  \begin{equation*}
    \begin{aligned}
       \tau 
       & \leq
       \brr{2L^2 + \frac{5C\lambda_2}{\alpha_{\rm P}}}5d\gamma
        \leq 
        \frac{45\lambda_2}{\alpha_{\rm D}\alpha_{\rm P}} { 5d\gamma} 
        \\ & =
        \cO\brr{   \frac{\brr{1-\frac{\alpha_{\rm D}}{2} }\brr{1-\frac{\alpha_{\rm P}}{2}}}{qw{\alpha_{\rm D}\alpha_{\rm P}}\brr{1  - \alpha_{\rm P}}\brr{1  - \alpha_{\rm D}}} \bar{L}  d\gamma}
         \\ &  
         = \cO\brr{  \frac{\bar{L}  d\gamma}{qw{\alpha_{\rm D}\alpha_{\rm P}}} }.
    \end{aligned}
  \end{equation*}

\section{Proofs of the lemmas}

\subsection{Proof of \Cref{lem:fp-transformation}}\label{sec:lem-fp-trans}
   Let $\rho_{0t}$ denote the joint distribution of $\left(y_0,\xi,y_t\right)$, which we write in terms of the conditionals and marginals as
  \begin{equation*}
  \rho_{0 t}\left(z,y_{0},\xi\right)=\rho_{0}\left(y_{0},\xi\right) \rho_{t \mid 0}\left(z \mid y_{0},\xi\right)=\rho_{t}\left(z\right) \rho_{0 \mid t}\left(y_{0},\xi \mid z\right) .
  \end{equation*}
  Conditioning on $\left(y_{0},\xi\right)$, the drift vector field $f_{\xi_k}(y_0)$ is a constant, so the Fokker-Planck formula for the conditional density $\rho_{t \mid 0}\left(z \mid y_{0},\xi\right)$ is given by
  \begin{equation*}
    \frac{\partial \rho_{t \mid 0}\left(z \mid y_{0},\xi\right)}{\partial t}=\nabla_z \cdot\left(\rho_{t \mid 0}\left(z \mid y_{0},\xi\right) f_{\xi}\left(y_{0}\right)\right)+\Delta \rho_{t \mid 0}\left(z \mid y_{0},\xi \right).
  \end{equation*}
  To derive the evolution of $\rho_{t}$, we integrate w.r.t. $\left(y_{0},\xi\right) \sim \rho_{0}$:
  \begin{equation}
    \begin{aligned}
      \frac{\partial \rho_{t}(z)}{\partial t} &=\int_{\mathbb{R}^{d}\times\Xi} \frac{\partial \rho_{t \mid 0}\left(z \mid y_{0},\xi\right)}{\partial t} \rho_{0}\left(y_{0},\xi\right) \rmd y_{0} \rmd \xi\\
      &=\int_{\mathbb{R}^{d}\times\Xi}\left(\nabla_z \cdot\left(\rho_{t \mid 0}\left(z \mid y_{0},\xi\right) f_{\xi}\left(y_{0}\right)\right)+\Delta \rho_{t \mid 0}\left(z \mid y_{0},\xi \right)\right) \rho_{0}\left(y_{0},\xi\right) \rmd y_{0} \rmd \xi \\
      &=\int_{\mathbb{R}^{d}\times\Xi}\left(\nabla_z \cdot\left(\rho_{0t}\left(z, y_{0},\xi\right)  f_{\xi}\left(y_{0}\right)\right)+\Delta \rho_{0t}\left(z, y_{0},\xi\right)\right) \rmd y_{0} \rmd \xi \\
      &=\nabla_z \cdot\left(\rho_{t}(z) \int_{\mathbb{R}^{d}\times\Xi} \rho_{0 \mid t}\left(y_{0} ,\xi\mid z\right) f_{\xi}\left(y_{0}\right) \rmd y_{0} \rmd \xi\right)+\Delta \rho_{t}(z) \\
      &=\nabla_z \cdot\left(\rho_{t}(z) \mathbb{E}_{\rho_{0 \mid t}}\left[f_{\xi}\left(y_{0}\right) \mid y_{t}=z\right]\right)+\Delta \rho_{t}(z).
    \end{aligned}
  \end{equation}
  Writing down the definition of KL divergence and using Fubini's theorem, we deduce
  \begin{equation}
    \begin{aligned}
      \dif{\KL{\rho_t}}
      &=\int_{\RR^d} \frac{\partial \rho_{t}(z)}{\partial t} \log\brr{\frac{\rho_t}{\pi}}(z)\rmd z\\
      &=\int_{\RR^d}\left(\nabla_z \cdot\left(\rho_{t}(z) 
      \mathbb{E}_{\rho_{0 \mid t}}\left[f_{\xi}\left(y_{0}\right) \mid y_{t}=z\right]\right)+\Delta \rho_{t}(z)\right)\log\brr{\frac{\rho_t}{\pi}}(z)\rmd z\\
      &=-\int_{\RR^d}\inner{\mathbb{E}_{\rho_{0 \mid t}}\left[f_{\xi}\left(y_{0}\right) \mid y_{t}=z\right]+\nabla\log(\rho_t)(z)}{\nabla\log\brr{\frac{\rho_t}{\pi}}(z)}\rho_t(z)\rmd z\\
      &=-\int_{\RR^d}\Big({\nabla\log\brr{\frac{\rho_t}{\pi}}(z)-\nabla\log\brr{\frac{\rho_t}{\pi}}(z)+\mathbb{E}_{\rho_{0 \mid t}}\left[f_{\xi}\left(y_{0}\right) \mid y_{t}=z\right]+\nabla\log(\rho_t)(z)}\Big)^{\top}\\ 
      & \hspace{2cm} \times
      {\nabla\log\brr{\frac{\rho_t}{\pi}}(z)}\rho_t(z)\rmd z\\
      &=-\int_{\RR^d}\inner{\nabla\log\brr{\frac{\rho_t}{\pi}}(z)+\mathbb{E}_{\rho_{0 \mid t}}\left[f_{\xi}\left(y_{0}\right) \mid y_{t}=z\right]-\nabla F(z)}{\nabla\log\brr{\frac{\rho_t}{\pi}}(z)}\rho_t(z)\rmd z.\\
        \end{aligned}
  \end{equation}
  We recall the definition of Fisher information to bound the first term of the scalar product:
  \begin{equation}
    \begin{aligned} 
      \dif{\KL{\rho_t}}
      &\leq-\FS{\rho_t}-\int_{\RR^d}\inner{\mathbb{E}_{\rho_{0 \mid t}}\left[f_{\xi}\left(y_{0}\right) \mid y_{t}=z\right]-\nabla F(z)}{\nabla\log\brr{\frac{\rho_t}{\pi}}(z)}\rho_t(z)\rmd z.
    \end{aligned}
  \end{equation}
  From the Cauchy-Schwartz inequality, we deduce 
  \begin{equation}
    \begin{aligned} 
      \dif{\KL{\rho_t}}
      &\leq -\FS{\rho_t}+\frac{1}{4}\FS{\rho_t} + 
      \int_{\RR^d}\norm{\mathbb{E}_{\rho_{0 \mid t}}\left[f_{\xi}\left(y_{0}\right) \mid y_{t}=z\right]-\nabla F(z)}^2\rho_t(z)\rmd z\\
      & = -\frac{3}{4}\FS{\rho_t}+\Exp{\normsq{\mathbb{E}\left[f_{\xi_k}(y_0)-\nabla F(y_t)\mid y_t\right]}}\\
      & \leq-\frac{3}{4}\FS{\rho_t}+\Exp{\mathbb{E}\left[\normsq{f_{\xi_k}(y_0)-\nabla F(y_t)}\mid y_t\right]}\\
      & = -\frac{3}{4}\FS{\rho_t}+\Exp{\normsq{f_{\xi_k}(y_0)-\nabla F(y_t)}}.
    \end{aligned}
  \end{equation}
    This  concludes the proof of the lemma.

  \subsection{Proof of \Cref{lem:dif-kl-fi}}
  If we replace $f_{\xi_k}(y_0)$ by $g_0$ in \eqref{eq:ef21-lmc-tytytyti},  we will have
  \begin{equation*}
    % \label{eq:ef21-lmc-KLM}
    \begin{aligned}
      \dif{\KL{\rho_t}}&\leq -\frac{3}{4}\FS{\rho_t}+\Exp{\normsq{\nabla F(y_t)-g_0}}\\
      &\leq -\frac{3}{4}\FS{\rho_t}+2\Exp{\normsq{\nabla F(y_t)-\nabla F(y_0)}}+2\Exp{\normsq{\nabla F(x_0)-g_0}}\\
      &=-\frac{3}{4}\FS{\rho_t}+2{\Exp{\normsq{\nabla F(y_t)-\nabla F(x_0)}}}
      +2\Exp{\normsq{\frac{1}{n}\sum_{i=1}^{n} \brc{\nabla F_i(x_0)-g^i_0}}}\\
      &\leq-\frac{3}{4}\FS{\rho_t}+2{\Exp{\normsq{\nabla F(y_t)-\nabla F(x_0)}}}+2\bG^{\rm D}_0.
    \end{aligned}
  \end{equation*}
  Here the last implication is due to Jensen's inequality.
  Let us bound the second term. %, denote $\cF^k_t$ the filtration generated by $\{B_{s}\}_{s=0}^{k\gamma+t}$, then
  The smoothness of the gradient yields
  \begin{equation}
    \label{eq:grad_ytx0-p1}
    \begin{aligned}
      \mathbb{E}\left[\normsq{\nabla F(y_t)-\nabla F(x_0)}\right]
      \leq L^2\mathbb{E}\left[\normsq{ y_t-x_0}\right]=L^2\mathbb{E}\left[\normsq{t g_0+\sqrt{2}\left(B_{t}-B_{0}\right)}\right].
    \end{aligned}
  \end{equation}
  Since the Brownian process has independent increments we get
  \begin{equation}
    \label{eq:grad_ytx0-p2}
    \begin{aligned}
      \mathbb{E}\left[\normsq{\nabla F(y_t)-\nabla F(x_0)}\right]
     &\leq L^2t^2\normsq{g_0}+2tL^2d \\
     &\leq L^2\gamma^2\normsq{g_0}+2hL^2d\\
     &= L^2\mathbb{E}\left[\normsq{ x_{1}-x_0}\right].
    \end{aligned}
  \end{equation}
  This concludes the proof.

  \subsection{Proof of \Cref{lem:exp_x_k+1_x_k}}

Let us apply \Cref{lem:Chew} to bound the term $\Exp{\normsq{x_{k+1}-x_k}}$:
\begin{equation*}
  \begin{aligned}
    \Exp{\normsq{x_{k+1}-x_k}}&=\gamma^2\Exp{\normsq{g_k}}+2d\gamma\\
    & \leq 2\gamma^2\left(\Exp{\normsq{\nabla F(x_k)}} + \Exp{\normsq{\nabla F(x_k)-g_k}}\right)+2d\gamma\\
    & \leq 2\gamma^2 \Exp{\normsq{\nabla F(x_k)}}+2\gamma^2\bG^{\rm D}_k+2d\gamma\\
    & \leq 4\gamma^2\left(\Exp{\norm{\nabla F(y_t)}}+\Exp{\normsq{\nabla F(y_t)-\nabla F(x_k)}}\right)+2\gamma^2 
    \bG^{\rm D}_k+2d\gamma\\
    & \leq 4\gamma^2\Exp{\norm{\nabla F(y_t)}}+4L^2\gamma^2\Exp{\normsq{x_{t}-x_k}}+2\gamma^2\bG^{\rm D}_k+2d\gamma\\
    & \leq 4\gamma^2\Exp{\norm{\nabla F(y_t)}}+4L^2\gamma^2\Exp{\normsq{x_{k+1}-x_k}}+2\gamma^2\bG^{\rm D}_k+2d\gamma.
  \end{aligned}
\end{equation*}
Regrouping the terms we obtain 
\begin{equation*}
  \begin{aligned}
    (1 - 4L^2\gamma^2)\Exp{\normsq{x_{k+1}-x_k}}
    & \leq 4\gamma^2\Exp{\norm{\nabla F(y_t)}} + 2\gamma^2\bG^{\rm D}_k + 2d\gamma.
  \end{aligned}
\end{equation*}
Dividing both sides on $1 - 4L^2\gamma^2$ and recalling that $2\sqrt{2}L\gamma < 1$, we conclude the proof.

\subsection{Proof of \Cref{lem:gamma-cond}}\label{proof:lem:gamma-cond}

    Is sufficient to show that 
    \begin{equation*}
      \gamma^2 \leq \min\brc{\frac{1}{192 L^2 },   \frac{\alpha_{\rm P}}{240C\lambda_2}}.
    \end{equation*}
    From the assumption of the theorem, we know that $\gamma^2 \leq \frac{1}{192 L^2 }$. Thus it remains to show that $\gamma^2$ is bounded by the minimum of the other two terms:
    \begin{equation*}
      \begin{aligned}
        \gamma^2 
        &\leq \frac{\alpha_{\rm P}}{240 C\lambda_2} = \frac{\alpha_{\rm P}\brr{e^{-\mu \gamma} -\lambda_1}}{510\lambda_2}.
      \end{aligned}
    \end{equation*}

    % Let us now compare the two denominators.
   
    % Thus, it is sufficient to prove 
    % \begin{equation*}
    %   \begin{aligned}
    %     \gamma^2 
    %     &\leq \frac{e^{-\mu \gamma} -\lambda_1}{204\lambda_2} \min\brc{1, { {e^{-\mu \gamma}  - (1-\alpha_{\rm P}) (1+w) }} }
    %     \\ 
    %     &= \frac{e^{-\mu \gamma} -\lambda_1}{204\lambda_2} \cdot  \brr{e^{-\mu \gamma}  - (1-\alpha_{\rm P}) (1+w) }.
    %   \end{aligned}
    % \end{equation*}
    Since $u = 1$ and $s=q$ we have the following bound on $\lambda_2$:
    \begin{equation*}
      \begin{aligned}
        \lambda_2 
        &\leq  \brs{2(1+q)(1+q^{-1})
             + \brr{2(1+q)(1+q^{-1})+ (1+q^{-1})} (1 + w^{-1})}\bar{L} \\
        &=  \brs{2(2+q+q^{-1})
             + \brr{2(2+q+q^{-1}) + (1+q^{-1})} (1 + w^{-1})}\bar{L} \\
        &= \frac{1}{q} \brs{2(2q+q^2+1)
             + \brr{2(2q+q^2+1) + (q+1)} (1 + w^{-1})}\bar{L} \\
        &\leq \frac{1}{qw} { {5(q+1)^2} (1 + w)}\bar{L} \\
        &\leq \frac{5}{qw}  \frac{\brr{1-\frac{\alpha_{\rm D}}{2} }\brr{1-\frac{\alpha_{\rm P}}{2}}}{\brr{1  - \alpha_{\rm P}}\brr{1  - \alpha_{\rm D}}} \bar{L}. 
      \end{aligned}
    \end{equation*}

    Therefore, we have an upper bound on $\lambda_2$. This means that it is sufficient for us to prove 
     \begin{equation*}
      \begin{aligned}
        \gamma^2 
        &\leq \frac{\alpha_{\rm P}\brr{e^{-\mu \gamma} -\lambda_1}}{510\frac{5}{qw}  \frac{\brr{1-\frac{\alpha_{\rm D}}{2} }\brr{1-\frac{\alpha_{\rm P}}{2}}}{\brr{1  - \alpha_{\rm P}}\brr{1  - \alpha_{\rm D}}} \bar{L}} 
        & = \frac{{qw}\alpha_{\rm P}\brr{e^{-\mu \gamma} -\lambda_1}}
        {2550 \bar{L}} 
        \cdot\frac{\brr{1  - \alpha_{\rm P}}\brr{1  - \alpha_{\rm D}}}{\brr{1-\frac{\alpha_{\rm D}}{2} }\brr{1-\frac{\alpha_{\rm P}}{2}}}.
      \end{aligned}
    \end{equation*}
    From $\mu\gamma < \min\brc{\alpha_{\rm D},\alpha_{\rm P}}/4$ and $e^t > 1 + t$, we deduce
    $e^{-\mu \gamma} -\lambda_1 > \alpha_{\rm D}/4$.
    Combining these inequalities with \eqref{eq:qw}, we deduce that it is sufficient to prove
     \begin{equation*}
      \begin{aligned}
        \gamma^2 
        &\leq  \frac{qw\alpha_{\rm D}\alpha_{\rm P} \brr{1  - \alpha_{\rm P}}\brr{1  - \alpha_{\rm D}}}{10200\brr{1-\frac{\alpha_{\rm D}}{2} }\brr{1-\frac{\alpha_{\rm P}}{2}} \bar{L}}.
      \end{aligned}
    \end{equation*}
    Finally, using \eqref{eq:qw} once again, we derive 
    \begin{equation*}
      qw \geq \frac{\alpha_{\rm P}\alpha_{\rm D}}{24(1-\alpha_{\rm P})(1-\alpha_{\rm D})}.
    \end{equation*}
    Therefore, 
     \begin{equation*}
      \begin{aligned}
        \gamma^2 
        &\leq  \frac{\alpha_{\rm D}^2\alpha_{\rm P}^2 }{244800\brr{1-\frac{\alpha_{\rm D}}{2} }\brr{1-\frac{\alpha_{\rm P}}{2}} \bar{L}}.
      \end{aligned}
    \end{equation*}
    Taking square root on both sides we obtain
     \begin{equation*}
      \begin{aligned}
        \gamma 
        &\leq  \frac{\alpha_{\rm D}\alpha_{\rm P} }{495\sqrt{\brr{1-\frac{\alpha_{\rm D}}{2} }\brr{1-\frac{\alpha_{\rm P}}{2}} \bar{L}}}.
      \end{aligned}
    \end{equation*}
    This concludes the proof.

\end{document}